\documentclass[twoside]{article}

\usepackage[accepted]{aistats2023}
\usepackage[round]{natbib}

\usepackage{amsmath}
\usepackage{amsfonts}
\usepackage{amssymb}
\usepackage{amsthm}
\usepackage{xcolor}
\usepackage{xurl}
\usepackage{hyperref}
\usepackage{enumerate}
\usepackage{graphicx}

\usepackage{algorithm,algpseudocode}
\algnewcommand\algorithmicinput{\textbf{Input:}}
\algnewcommand\Input{\item[\algorithmicinput]}
\algnewcommand\algorithmicoutput{\textbf{Output:}}
\algnewcommand\Output{\item[\algorithmicoutput]}

\theoremstyle{plain}

\newtheorem{theorem}{Theorem}[section]
\newtheorem{lemma}{Lemma}[section]
\newtheorem{proposition}{Proposition}[section]
\newtheorem{corollary}{Corollary}[section]

\numberwithin{equation}{section}

\theoremstyle{definition}
\newtheorem{definition}{Definition}[section]
\newtheorem{example}{Example}[section]

\theoremstyle{remark}

\DeclareMathOperator{\trace}{Tr}
\DeclareMathOperator{\rank}{rank}
\DeclareMathOperator{\diag}{diag}
\newcommand{\norm}[1]{\lVert#1\rVert}

\begin{document}



\twocolumn[
\aistatstitle{Scalable Spectral Clustering with Group Fairness Constraints}
\aistatsauthor{
Ji Wang \And Ding Lu \And  Ian Davidson \And Zhaojun Bai }
\aistatsaddress{
Univ. of California, Davis \\
{\tt jiiwang@ucdavis.edu}
\And   
University of Kentucky \\
{\tt Ding.Lu@uky.edu}
\And 
Univ. of California, Davis \\
{\tt indavidson@ucdavis.edu} 
\And 
Univ. of California, Davis \\
{\tt zbai@ucdavis.edu}
} 
]

\begin{abstract}
  There are synergies of research interests
and industrial efforts in modeling fairness and 
correcting algorithmic bias 
in machine learning. 
In this paper, we present a scalable algorithm for 
spectral clustering (SC) with group fairness constraints.
Group fairness is also known as statistical parity where in each cluster, 
each protected group is represented with 
the same proportion as in the entirety.
While FairSC algorithm \citep{kleindessner2019guarantees}
is able to find the fairer clustering, 
it is compromised by high computational costs due to the algorithm's kernels 
of computing nullspaces and the square roots of dense matrices
explicitly. We present a new formulation of the underlying spectral 
computation of FairSC by incorporating nullspace projection 
and Hotelling's deflation such that the resulting algorithm, 
called s-FairSC, only involves the sparse matrix-vector products 
and is able to fully exploit the sparsity of the fair SC model.
The experimental results on the modified stochastic block model 
demonstrate that while it is comparable with FairSC in recovering 
fair clustering, s-FairSC is 12$\times$ faster than FairSC for moderate model sizes.  s-FairSC is further demonstrated to be  
scalable in the sense that the computational costs of s-FairSC only increase marginally compared to the SC without fairness constraints.
\end{abstract}

\section{INTRODUCTION}

Machine learning (ML) is widely used to automate decisions 
in areas such as targeting of advertising, issuing 
of credit cards, and admission of students.
While powerful, ML is vulnerable to 
biases encoded in the raw data or brought by underlying algorithms 
against certain groups or individuals,
thus resulting in {\it unfair}  
decisions~\citep{hardt2016equality,chouldechova2018frontiers}.  
Examples of algorithmic unfairness in real life are documented in \cite{flores2016false, pethig2022biased}.
In the context of algorithmic decision-making, 
{\it fairness} commonly refers to the prohibition of any 
favoritism toward certain groups or individuals based on their 
natural or acquired characteristics. Such characteristics are also 
known as sensitive attributes, for instance, gender, ethnicity, 
sexual orientation, and age group~\citep{mehrabi2021survey}. 
The rising stake and growing societal impact of ML algorithms have 
motivated the study of fairness in academia and industry. 
Various efforts have been attempted at modeling fairness and 
correcting algorithmic biases in both supervised and unsupervised 
ML, see e.g., \cite{dwork2012fairness,chierichetti2017fair,samadi2018price,agarwal2019fair,aghaei2019learning,amini2019uncovering,zhang2019framework, davidson2020making}.

There is a wide range of studies in fair ML  
depending on the choice of algorithms and fairness definitions.
In this paper, we focus on spectral clustering (SC) with group 
fairness constraints.  The notion of group fairness is an idea 
of statistical parity by ~\cite{feldman2015certifying,zemel2013learning}. 
It ensures that the proportion of members in a group 
receiving positive (negative) consideration is identical to 
the proportion of the population as a whole.
In \cite{kleindessner2019guarantees},  
a mathematical model is proposed to incorporate the group fairness 
into the SC framework, FairSC for short.
For synthetic networks,  FairSC is shown to recover ground-truth clustering with high probability. For real-life datasets, 
FairSC identifies a fairer clustering compared to SC without fairness constraints.
Unfortunately, FairSC can only handle moderate model sizes 
due to the computational costs in the computations
of orthonormal bases of large nullspaces and the square roots 
of dense matrices. FairSC is not scalable.

In this paper, we present a new formulation of spectral computation of FairSC 
by incorporating nullspace projection and Hotelling's deflation. 
The resulting algorithm is named Scalable FairSC, or s-FairSC. 
In s-FairSC, all computational kernels only involve the sparse 
matrix-vector multiplications and therefore are capable of fully
exploiting the sparsity of the fair SC model.
A comparison of s-FairSC with FairSC on the modified stochastic block model exhibits 12x speed up for moderate model sizes. 
Meanwhile, s-FairSC is comparable with
FairSC in recovering fair clustering.
The s-FairSC is further demonstrated to be
scalable in the sense that it only has a marginal increase in
computational costs compared to the SC without fairness constraints.

The remainder of the paper is organized as follows.
Section~\ref{sec:fairSC} reviews the basics of spectral clustering and group fairness. 
Section~\ref{sec:algs} first recaps FairSC and then derives s-FairSC. 
Section~\ref{sec:expr} starts with the descriptions of experimental datasets and then demonstrates the improvements of s-FairSC 
in computational efficiency and scalability while maintaining 
the same accuracy as FairSC. Concluding remarks are in Section~\ref{sec:concl}.

\section{SC AND FAIR SC}\label{sec:fairSC}

\subsection{Clustering and fair clustering} 
Given a set of data, the goal of clustering is to partition 
the set into subsets such that data in the same subset 
is more similar to each other than in those of the other subsets. 
Mathematically, let $\mathcal{G}(V,W)$ denote a weighted and 
undirected graph with a set of vertices (data) 
$V=\{v_1,v_2,\dots,v_n\}$ and a {\it weighted adjacency matrix}
$W = (w_{ij}) \in \mathbb{R}^{n \times n}$.
The matrix $W$ encodes the edge information.
We assume $w_{ij}\geq 0$ and $w_{ii}=0$.
If $w_{ij}>0$, then $(v_i,v_j)$ is an edge with weight $w_{ij}$. 
We denote by $d_i = \sum_{j = 1}^{n} w_{ij}$
the {\it degree} of a vertex $v_i$ and 
$D=\mbox{diag}(d_1,d_2,\dots,d_n)$,
the {\it degree matrix} of $\mathcal{G}(V,W)$.
For simplicity, we assume there is no isolated vertex, and consequently, 
$D$ is positive definite.

The task of clustering is to partition $V$ into $k$ 
disjoint subsets ({\it clusters}):
\begin{equation} \label{eq:clusteringdef} 
V = C_1 \cup \cdots \cup C_k,
\end{equation} 
such that
the total weights within each subset are large and 
between two different subsets are small.
The clustering~\eqref{eq:clusteringdef} can be encoded in a {\it clustering indicator matrix} 
$H = (h_{i\ell}) \in \mathbb R^{n \times k}$, where for $i=1,\ldots,n$ and $\ell = 1,\ldots, k$, 
\begin{equation} \label{eq:H}
	h_{i\ell} :=  \left\{
                \begin{array}{ll}
                  1, & \mbox{if } v_i \in C_{\ell}, \\
                  0, & \mbox{ otherwise.}
                \end{array}
              \right.
\end{equation}  
Now let us consider how to enforce {\it group fairness} in 
clustering.  The {\it groups} refer to a partition of 
the collected data $V$
(e.g., based on sensitive attributes such as gender and race).
We denote the groups with $h$ non-empty subsets:
\begin{equation} \label{eq:groupdef}
V = V_{1}\cup \cdots\cup V_{h},
\end{equation} 
where $V_i \cap V_j=\emptyset$ for $i\neq j$.
Groups can be encoded in a {\it group indicator matrix} 
$G = (g_{is}) \in \mathbb R^{n \times h}$, where for $i=1,\dots,n$ and $s=1,\dots, h$
\begin{equation} \label{eq:group-vec}
	g_{is}
	:=  \left\{
        \begin{array}{ll}
          1, & \mbox{if }v_i \in V_{s}, \\
          0, & \mbox{ otherwise.}
        \end{array}
      \right.
\end{equation}
The {\it group fairness} for clustering refers to the case that 
objects from all groups are presented proportionately in each cluster,
also known as statistical parity.
The following definition is due to~\cite{kleindessner2019guarantees},
which extends the notion of group fairness 
by~\cite{chierichetti2017fair}.  

\begin{definition} \label{def:fair}
A clustering~\eqref{eq:clusteringdef}
is {\it group fair} with respect to 
a group partition~\eqref{eq:groupdef}
if in each cluster the objects from each group are presented proportionately 
as in the original dataset. 
That is, for $s = 1, 2, \cdots, h$ and $\ell = 1, 2, \cdots, k$, 
\begin{equation} \label{eq:fair}
\frac{|V_{s}\cap C_{\ell}|}{|C_{\ell}|} 
= \frac{|V_{s}|}{|V|},
\end{equation} 
where $|V|$ denotes the number of vertices in $V$.
\end{definition}

The fairness condition~\eqref{eq:fair} can be represented 
compactly using the matrices $H$ in \eqref{eq:H}
and $G$ in \eqref{eq:group-vec}. To do so,
let us first introduce matrices 
\begin{equation}\label{eq:mz}
M := G^T H
\quad\text{and}\quad
Z := (G^T {\bf 1}_n)\cdot (H^T{\bf 1}_n )^T, 
\end{equation}
where ${\bf 1}_n$ is a length-$n$ column
vector with all elements equal to 1. 
Then the entries of $M$ and $Z$ are 
$m_{s\ell} = |V_s\cap C_\ell|$  
and $z_{s\ell} = |V_s|\cdot |C_\ell|$
for $s=1,2,\dots, h$ and $\ell = 1,2,\dots, k$.
Consequently, the fairness condition~\eqref{eq:fair} is equivalent to 
\begin{equation}\label{eq:nmz}
	n\cdot M = Z,
\end{equation}
where $n = |V|$.
By~\eqref{eq:mz}, equation~\eqref{eq:nmz} holds if and only if 
\begin{equation}\label{eq:f0h}
F_0^T H = 0,
\end{equation}
where 
$F_0 := G - {\bf 1}_n z^T \in\mathbb R^{n\times h}$
and $z:=(G^T{\bf 1}_n)/n \in\mathbb R^{h}$.
Observe that according to the definition of $G$ in~\eqref{eq:group-vec}, 
the entries of the vector $z$ satisfy
$z_i = |V_i|/n$, for $i=1,2,\dots, h$.

The following lemma 
shows that it is sufficient to use the first $h-1$
columns of $F_0$ in the constraint~\eqref{eq:f0h}. 
The idea of using the first $h-1$ columns of $F_0$
is from~\cite{kleindessner2019guarantees}.
Extended from this idea, we justify the choice of $h-1$ through the rank of $F_0$ and prove that $h-1$ is indeed the least number of columns necessary.

\begin{lemma}\label{le:1}
Let $H$ be the clustering indicator matrix in~\eqref{eq:H},
and $G$ be the group indicator matrix as in~\eqref{eq:group-vec}.
Let $F_0\in \mathbb R^{n\times h}$ be as defined in~\eqref{eq:f0h} and 
$F:=F_0(:,1:h-1)\in \mathbb R^{n\times (h-1)}$ consists of the first
$h-1$ columns\footnote{By changing the order of groups $\{V_s\}$, 
the result also holds for $F$ with arbitrary $(h-1)$ columns of $F_0$.}
of $F_0$.
Then, 
\begin{enumerate}[(i)]
    \item \label{i:le:1:1}
    $\rank{(F_0)} = \rank{(F)} = h-1$.
    
    \item \label{i:le:1:2}
    The clustering~\eqref{eq:clusteringdef} is group fair 
		with respect to~\eqref{eq:groupdef}
		if and only if
		\begin{equation} \label{eq:fairmath}
		F^T H = 0.
		\end{equation}
\end{enumerate} 
\end{lemma}
\begin{proof}
See Appendix~\ref{appx-le2-1}.
\end{proof}

\subsection{SC and fair SC}

\paragraph{SC.}
The objective function of a normalized cut (NCut) \citep{shi2000normalized,ng2001spectral}
 is
\begin{equation} \label{eq:ncut}
    \mbox{NCut}(C_1, \cdots, C_k) := \sum_{\ell=1}^{k}\frac{\mbox{Cut}(C_{\ell}, V\backslash C_{\ell})}{\mbox{vol}(C_{\ell})},
\end{equation}
where
\[
\mbox{Cut}(C_\ell, V\backslash C_\ell) 
= \sum_{
\substack{v_i \in C_\ell \\
v_j \in V\backslash C_\ell}
} w_{ij}, \,\, 
\mbox{vol}(C_\ell) = \sum_{v_i \in C_\ell}d_i.
\]
The NCut function calculates the scaled total weights of {\it between-cluster}
edges, and  measures the similarities between the clusters:
a smaller NCut value implies better clustering.
The scaling in~\eqref{eq:ncut} by \mbox{vol}$(C_{\ell})$ takes into account the size of the cluster to avoid outliers. 
Hence, the goal is to minimize NCut.

The NCut function~\eqref{eq:ncut} admits a nice expression 
using the clustering indicator matrix $H$. 
Let us first scale the clustering indicator matrix
$H$ in \eqref{eq:H} to 
\begin{equation} \label{eq:normal-H}
H \leftarrow H \widehat D^{-1},
\end{equation}
where $\widehat{D} = \diag{(\sqrt{\mbox{vol}(C_1)}, \cdots,
\sqrt{\mbox{vol}(C_k)})}$.
We call the new $H$ the {\it scaled indicator matrix}.
For convenience, we use the 
same notation for both scaled and unscaled cluster indicator matrices.
Then, the NCut function \eqref{eq:ncut} is recast to the following matrix trace:
\begin{equation}\label{eq:trace}
\mbox{NCut}(C_{1}, \cdots, C_{k}) = \trace{(H^{T}LH)},
\end{equation}
where $L = D - W$ is the {\it Laplacian} of $\mathcal{G}(V,W)$.
Note that under the assumption of connectivity of $\mathcal{G}(V,W)$, 
$L$ is semi-positive definite 
and has exactly one zero eigenvalue.
By~\eqref{eq:trace}, the NCut minimization is equivalent to the trace minimization problem
\begin{equation} \label{eq:normal-prob}
	\min{\trace{(H^{T}LH)}} 
	\quad \mbox{s.t.} \quad \mbox{$H$ is of the form~\eqref{eq:normal-H}}.
\end{equation}
Solving problem~\eqref{eq:normal-prob} directly is
NP-hard~\citep{wagner1993between}.
In practice, the following relaxed version of the problem~\eqref{eq:normal-prob}
is solved:
\begin{equation} \label{eq:normal-prob1}
	\min_{H \in \mathbb{R}^{n \times k}}{\trace{(H^{T}LH)}} 
	\quad \mbox{s.t.} \quad H^{T}DH = I_{k}.
\end{equation}
Once an optimal solution $H$ of \eqref{eq:normal-prob1} is obtained, 
a discrete solution of~\eqref{eq:normal-prob} 
can be obtained by a properly chosen criterion.
Subsequently, the $k$-means algorithm (for the rows of $H$) is applied for clustering, 
although other techniques are also available; see, e.g.,~\cite{bach2003learning, lang2005fixing}.

Problem~\eqref{eq:normal-prob1} is a classical trace minimization problem initially studied in~\cite{fan1949theorem}.
The following theorem can be found in~\cite[p.~248]{horn2012matrix}.

\begin{theorem}  \label{th:mintr}
For a symmetric matrix $A\in \mathbb{R}^{n \times n}$, 
\begin{equation*}
\min_{X^{T}X = I_{k}}\trace{(X^{T}AX)} =
\trace{(X_*^{T}AX_*)} =  \sum_{i=1}^{k}\lambda_i,
\end{equation*}
where
$\lambda_1 \leq \cdots \leq \lambda_k$ are the $k$ smallest eigenvalues of $A$,
and columns of $X_* \in \mathbb{R}^{n \times k}$ are the corresponding eigenvectors.
\end{theorem}

By Theorem~\ref{th:mintr}, we can reformulate
problem~\eqref{eq:normal-prob1}
to the standard trace minimization problem by 
a change of variables $X=D^{1/2} H$:
\begin{equation} \label{eq:normal-prob2}
	\min_{X\in\mathbb R^{n\times k}}{\trace{(X^{T}L_{\rm n}X)}}
	\quad \mbox{s.t.} \quad X^{T} X = I_k,
\end{equation}
where
$L_{\rm n} = D^{-\frac{1}{2}}LD^{-\frac{1}{2}}$, 
which
is known as the {\it normalized Laplacian}, and then 
compute the eigenvectors $X$ corresponding to the $k$ smallest eigenvalues of $L_{\rm n}$. 
The solution of the problem~\eqref{eq:normal-prob1} is
recovered by $H=D^{-1/2}X$. The SC algorithm is summarized in Algorithm~\ref{code:algo1}.

\begin{algorithm}[h]
    \caption{SC (Spectral Clustering)}
    \label{code:algo1}
\begin{algorithmic}[1]
    \Input  weighted adjacency matrix $W \in \mathbb{R}^{n \times n}$;
    degree matrix $D \in \mathbb{R}^{n \times n}$;
    $k \in \mathbb{N}$

    \Output a clustering of indices $1:n$ into $k$ clusters
	\State compute the Laplacian matrix $L = D - W$;
	
	\State compute the normalized Laplacian $L_{\rm n} =
	D^{-\frac{1}{2}}LD^{-\frac{1}{2}}$;
	
	\State compute the $k$ smallest eigenvalues of $L_{\rm n}$ 
	and the corresponding eigenvectors $X \in \mathbb{R}^{n \times k}$;
	
	\State apply $k$-means clustering to the rows of $H = D^{-\frac{1}{2}}X$.
\end{algorithmic}
\end{algorithm}
The SC~\citep{shi2000normalized,ng2001spectral} 
is a highly successful clustering algorithm, and widely used in areas of data exploration, 
such as image
segmentation~\citep{tung2010enabling}, 
speech separation~\citep{bach2006learning} among others. 
The SC algorithm is efficient and scalable 
because it can fully take the advantage of sparsity of the SC model and use the state-of-the-art scalable sparse eigensolvers~\citep{bddrv:2000}. 

\paragraph{Fair SC.}
The group fairness constraints can be elegantly incorporated into the SC by simply adding the constraint~\eqref{eq:fairmath}
to the trace minimization problem~\eqref{eq:normal-prob1},
which leads to 
\begin{equation} \label{eq:normal-prob-fair}
	\min_{H \in \mathbb{R}^{n \times k}}{\trace{(H^{T}LH)}} 
	~~\mbox{s.t.}~~ H^{T}DH = I_{k} \mbox{ and }  F^{T}H = 0, 
\end{equation}
where $L \in \mathbb{R}^{n \times n}$ is the graph Laplacian, 
$F \in \mathbb{R}^{n \times (h-1)}$ is from~\eqref{eq:fairmath},
and $F^TH=0$ is the original~\eqref{eq:fairmath} right-multiplied 
with $\widehat D^{-1}$ due to the scaling~\eqref{eq:normal-H} of $H$.

The idea of enforcing group fairness in spectral clustering
using the optimization~\eqref{eq:normal-prob-fair}
was proposed in~\cite{kleindessner2019guarantees}.
We will show that the additional
linear constraints in~\eqref{eq:normal-prob-fair}
will introduce only marginal extra costs than solving the SC \eqref{eq:normal-prob1}.

\section{ALGORITHMS}\label{sec:algs}
In this section, we consider numerical algorithms for solving the constrained
trace minimization problem~\eqref{eq:normal-prob-fair}.
We first review the FairSC algorithm proposed in~\cite{kleindessner2019guarantees} and then address the scalability issue of the FairSC.

\subsection{FairSC algorithm}  
A nullspace-based algorithm for solving
the problem~\eqref{eq:normal-prob-fair} proposed 
in \cite{kleindessner2019guarantees} is as follows.  
Since the columns of $H$ live in the nullspace of $F^T$, we can write 
$$H = ZY$$
for some  $Y\in\mathbb{R}^{(n-h+1)\times k}$,
where $Z \in \mathbb{R}^{n \times (n-h+1)}$ 
is an orthonormal basis matrix of $\mbox{null}(F^{T})$.
Consequently, the optimization problem~\eqref{eq:normal-prob-fair} 
is equivalent to the following trace optimization
without linear constraints:
\begin{equation} \label{eq:normal-prob-fair-1}
\min_{Y \in \mathbb{R}^{(n-h+1) \times k}}{\trace{(Y^{T}[Z^{T}LZ]Y)}} 
~\mbox{s.t.}~ Y^{T}[Z^{T}DZ]Y = I_{k}.
\end{equation}
We can further transform the problem~\eqref{eq:normal-prob-fair-1} to 
the standard trace minimization~\eqref{eq:normal-prob1} 
by another change of variables 
$$X= QY 
\quad \mbox{with} \quad Q=(Z^TDZ)^{1/2},
$$
which leads to
\begin{equation} \label{eq:normal-prob-fair-2}
    \min_{X \in \mathbb{R}^{(n-h+1) \times k}}{\trace{(X^{T}MX)}} 
    \quad \mbox{s.t.} \quad X^{T}X = I_{k},
\end{equation}
where $M = Q^{-1}Z^{T}LZQ^{-1}$. 
Observe that $M$ is positive semi-definite of size $n-h+1$.
According to Theorem~\ref{th:mintr},
problem~\eqref{eq:normal-prob-fair-2}
is solved by linear eigenvalue problem $Mx =\lambda x$. 
The optimal solution $X=[x_1,\dots,x_k]$ consists of the 
eigenvectors corresponding to the $k$ smallest eigenvalue of $M$. 
Finally, 
$H = Z Q^{-1}X$
is the solution of the fair SC minimization~\eqref{eq:normal-prob-fair}.

We summarize the aforementioned algorithm
for the group-fair spectral clustering in Algorithm~\ref{code:algo2},
called FairSC.
FairSC requires two major computational kernels.
The first one is the nullspace of $F^T$ in step~2.
This can be done by the SVD of $F=U\Sigma V^T$,
where $U \in \mathbb{R}^{n\times n}$ and $V \in \mathbb{R}^{(h-1)\times (h-1)}$
are orthogonal, and $\Sigma \in \mathbb{R}^{n\times (h-1)}$ is diagonal.
According to Lemma~\ref{le:1}, $F$ has a full column rank $h-1$.
Therefore $U(:,h:n)$ is an orthonormal basis of the nullspace of $F^T$.
We can also use QR decomposition 
$F=UR$, where $U \in \mathbb{R}^{n\times n}$ is orthogonal
and $R \in \mathbb{R}^{n\times (h-1)}$ is upper triangular.
We can then set $Z=U(:,h:n)$.
For both SVD and QR, the computation complexity is about 
$\mathcal O(n(h-1)^2)$; see,
e.g.,~\cite{golub1996matrix}.
The second kernel is the matrix square root of size
$n-h+1$ in step 3.
This can be done by the blocked Schur
algorithm~\citep{higham2010computing,deadman2012blocked}.
The computation complexity is $\mathcal O((n-h+1)^3)$.
For matrices of large sizes, both kernels involving large dense matrices are 
computationally expensive due to memory space and
data communication costs.
Consequently, FairSC is only suitable for small to medium size fair SC models;
see numerical results in Section~\ref{sec:expr}.

\begin{algorithm}[h]
    \caption{FairSC} \label{code:algo2}
    \begin{algorithmic}[1]
    
    \Input weighted adjacency matrix $W \in \mathbb{R}^{n \times n}$;
		%({\bf the graph must not have any isolated vertex}); 
    degree matrix $D \in \mathbb{R}^{n \times n}$;
    group-membership matrix $F \in \mathbb{R}^{n \times (h-1)}$;
    $k \in \mathbb{N}$
    
    \Output a clustering of indices $1:n$ into $k$ clusters

    \State compute the Laplacian matrix $L = D - W$;
    
    \State compute an orthonormal basis $Z$ of the nullspace of $F^{T}$;
    
    \State compute the matrix square root $Q= (Z^{T}DZ)^{1/2}$;
    
    \State compute $M = Q^{-1}Z^{T}LZQ^{-1}$;
    
    \State compute the $k$ smallest eigenvalues of $M$ and 
	the corresponding eigenvectors $X \in \mathbb{R}^{n \times k}$;
    
    \State apply $k$-means clustering to the rows of $H = ZQ^{-1}X$.
    \end{algorithmic}
\end{algorithm}

\subsection{First variant of FairSC}
As the first variant of FairSC, we can avoid computing the square root of a dense matrix by reordering the changes of variables used in FairSC. 
Let us begin with a change of variables 
$$
X = D^{\frac{1}{2}}H
$$
and turn the optimization~\eqref{eq:normal-prob-fair} to 
\begin{equation} \label{eq:tropt_null}
    \min_{X \in \mathbb{R}^{n \times k}}{\trace{(X^{T}L_{\rm n}X)}} 
    \quad \mbox{s.t.} \quad X^{T}X = I_{k} \mbox{ and }
	C^{T}X = 0,
\end{equation}
where $L_{\rm n} = D^{-\frac{1}{2}}LD^{-\frac{1}{2}}$
is the normalized Laplacian, and $C = D^{-\frac{1}{2}}F$.
Recall that the degree matrix $D$ is diagonal,
so generating $L_{\rm n}$ and $C$ requires only row and column scaling.
Next, we remove the linear constraints $C^TX=0$ in~\eqref{eq:tropt_null}
using the nullspace basis of $C^T$. Specifically, since the columns of $X$ live in the nullspace of $C^T$ we can parameterize 
$$
X = VY
\quad \mbox{for some $Y\in\mathbb{R}^{(n-h+1)\times k}$},
$$
where $V\in\mathbb{R}^{n\times (n-h+1)}$ 
is an orthonormal basis matrix of the nullspace of $C^{T}$.
Then the optimization \eqref{eq:tropt_null} is equivalent to 
the standard trace minimization 
\begin{equation}\label{eq:tracemin1}
\min_{Y\in\mathbb{R}^{(n-h+1)\times k}}\trace{(Y^{T} \, L^{\rm v}_{\rm n}\,  Y)} 
\quad \mbox{s.t.} \quad Y^{T} Y = I_k,
\end{equation}
where 
$L^{\rm v}_{\rm n} = V^T L_n V  \in \mathbb{R}^{(n-h+1)\times (n-h+1)}$.
Consequently, by Theorem~\ref{th:mintr}, we just need to solve
the symmetric eigenvalue problem 
\begin{equation}\label{eq:lveig}
L^{\rm v}_{\rm n} y = \lambda y.
\end{equation}
The eigenvectors corresponding to the $k$ smallest eigenvalues 
provide the solution $Y$ of~\eqref{eq:tracemin1},
by which we recover the solution $H = D^{-\frac{1}{2}}VY$ 
of the fair SC minimization problem~\eqref{eq:normal-prob-fair}.

Although this variant of FairSC 
avoids computing matrix square root of a dense matrix,
the other drawbacks of FairSC remain, namely 
explicit computation of the nullspace of $C^T$
and eigenvalue computation of the dense matrix $L^{\rm v}_{\rm n}$.

\subsection{Scalable FairSC algorithm}

We now show how to reformulate the eigenvalue
problem~\eqref{eq:lveig} to address the remaining pitfalls of FairSC.
We begin with the eigenvalue problem of 
$L^{\rm v}_{\rm n} $ in~\eqref{eq:lveig}:
\begin{equation*}
(V^{T}L_{\rm n}V)\,y = \lambda y.
\end{equation*}
A left multiplication of $V$ leads to 
\begin{equation}\label{eq:peig1}
(VV^{T} L_{\rm n}VV^T) Vy = \lambda\cdot Vy,
\end{equation}
where on the left side $VV^T\cdot Vy \equiv Vy$ due to the fact $V^T V = I$.
Denote by $P = V V^{T}$
a projection matrix onto the range space of $V$ (i.e., nullspace of $C^T$). 
Then~\eqref{eq:peig1} leads to the following {\it projected eigenvalue problem}
\begin{equation}\label{eq:eig2}
L^{\rm p}_{\rm n}\, x = \lambda x, 
\end{equation}
where $x = Vy$ and
$L^{\rm p}_{\rm n} = PL_{\rm n}P \in\mathbb R^{n\times n}$.
Consequently, an eigenvalue $\lambda$ of $L^{\rm v}_{\rm n}$ in~\eqref{eq:lveig}
is also an eigenvalue of $L^{\rm p}_{\rm n}$ in~\eqref{eq:eig2}. A major advantage of the projected eigenvalue problem~\eqref{eq:eig2} is
that it may avoid the computation of the nullspace of $C^T$
by exploiting the fact that the projection matrix
\begin{equation}\label{eq:pmat2}
	P =  I - U_2U_2^T,
\end{equation}
where $U_2\in\mathbb R^{n\times (h-1)}$ is an orthonormal basis of the range $C$.
($[V,U_2]\in\mathbb R^{n\times n}$ is orthogonal)
This is especially beneficial since $C$ is a tall and skinny matrix, where $U_2\in\mathbb R^{n\times (h-1)}$ is much smaller than
$V\in\mathbb R^{n\times (n-h+1)}$.
In addition, to compute the eigenvalues of $L^{\rm p}_{\rm n}$ 
by an iterative eigensolver, we only need the matrix-vector product $L^{\rm p}_n w$ for a given vector $w$ and the matrix
$L^{\rm p}_{\rm n}$ is never formed explicitly. 
The product $Pw$ can be applied without formulating $U_2$;
see implementation detail in Section~\ref{sec:impl-iss}.

For FairSC, we need the $k$ smallest eigenvalues of the matrix
$L^{\rm v}_{\rm n}$ in \eqref{eq:lveig}. 
The following proposition
connects the eigenstructures of 
the matrices $L^{\rm v}_{\rm n}$ and $L^{\rm p}_{\rm n}$.

\begin{proposition} \label{prop:1}
Suppose $L^{\rm v}_{\rm n}$ in \eqref{eq:lveig} has the eigendecomposition
	\begin{equation}\label{eq:eiglv}
		L^{\rm v}_{\rm n}= Y \Lambda_{\rm v}Y^T,
	\end{equation}
	where $\Lambda_{\rm v}=\mbox{diag}(\lambda_1,\lambda_2,\dots,\lambda_{n-h+1})$
	contains the eigenvalues, and $Y$ is an orthogonal matrix of order
	$n-h+1$ containing eigenvectors.
	Then the matrix $L^{\rm p}_{\rm n}$ from~\eqref{eq:eig2}
	has the eigendecomposition
	\begin{equation} \label{eq:blockAp} 
		L^{\rm p}_{\rm n} = \begin{bmatrix} U_1 & U_2 \end{bmatrix}
		\begin{bmatrix} \Lambda_{\rm v} & \\ & {\bf 0}_{h-1,h-1} \end{bmatrix}
		\begin{bmatrix} U^T_1 \\ U_2^T \end{bmatrix},
	\end{equation}
	where $U=[U_1,U_2]\in \mathbb R^{n\times n}$ is orthogonal
	with $U_1 = VY \in \mathbb R^{n\times (n-h+1)}$
	and $U_2\in\mathbb R^{n\times (h-1)}$ being an arbitrary orthonormal
	basis of the range of $C$.
\end{proposition}
\begin{proof}
See Appendix~\ref{appx:prop3-1}.
\end{proof}

The following is a direct consequence of Proposition~\ref{prop:1}.

\begin{corollary} 
Let $L^{\rm v}_{\rm n}$ and $L^{\rm p}_{\rm n}$ be defined as in~\eqref{eq:tracemin1} 
and~\eqref{eq:eig2}. Then
\begin{enumerate}[(i)]
\item
If $(\lambda, y)$ is an eigenpair of $L^{\rm v}_{\rm n}$, then
$(\lambda, x)$ with $x = Vy$ is an eigenpair of $L^{\rm p}_{\rm n}$.

\item
If $(\lambda, x)$ is an eigenpair of $L^{\rm p}_{\rm n}$ and $C^{T}x =
0$, then $(\lambda, y)$ with $y = V^{T}x$ is an eigenpair of
$L^{\rm v}_{\rm n}$.
\end{enumerate}
\end{corollary}

Let us return to the eigenvalue problem~\eqref{eq:lveig}.
Since the matrix $L^{\rm v}_{\rm n}$ is positive semi-definite, 
it has $n-h+1$ ordered eigenvalues $0\leq \lambda_1 \leq \cdots \leq \lambda_{n-h+1}$.
By~\eqref{eq:blockAp}, the projected matrix $L^{\rm p}_{\rm n}$ has $n$ 
ordered eigenvalues:
$
\underbrace{0 = \cdots = 0}_{h-1} \leq 
\lambda_1 \leq \cdots \leq \lambda_{n-h+1}$,
where the first $h-1$ zero eigenvalues (counting multiplicity)
have eigenvectors in the range of $C$.  

In the simple case of $\lambda_1 >  0$, 
the $k$ smallest eigenvalues
$\lambda_1,\dots,\lambda_k$ of $L^{\rm v}_{\rm n}$ corresponds to the $k$ smallest
{\it positive eigenvalues} of $L^{\rm p}_{\rm n}$. 
To find those eigenvalues, we can  
first compute $K= k+h-1$ smallest eigenvalues of
$L^{\rm p}_{\rm n}$ by an eigensolver, and then select the desired $k$ eigenpairs
corresponding to non-zero eigenvalues
(alternatively, select those eigenvalues with eigenvectors orthogonal to $C$). 
However, if $\lambda_1 =  0$ (or $\lambda_1\approx 0$), 
then this simple selection scheme does not work,
as the eigenvector corresponding to $\lambda_1=0$ is mixed 
(or numerically mixed) with the eigenspace of the $h-1$ zero eigenvalues.
This eigenspace mixing issue happens, in particular, if the solution is computed by an iterative
method with low accuracy.

To address the eigenspace mixing issue, we turn to the second major contribution of this work, namely a novel use of Hotelling's deflation. In the following, we first discuss 
Hotelling's deflation~\citep{hotelling1943some}, which
is also known as explicit external deflation and is suitable for high-performance computing; 
see, e.g.,~\cite{Parlett:1998,Yamazaki:2019}.
The main idea of Hotelling's deflation is summarized in the following proposition.

\begin{proposition} \label{prop:hotelling} 
	Let
	the eigenvalue decomposition of a symmetric matrix $A \in \mathbb{R}^{n \times n}$ be given by
	\begin{equation}\label{eq:meig}
		A  = Q \Lambda Q^{T} 
		= \begin{bmatrix}
		  Q_{1} & Q_{2}
		\end{bmatrix}
		\begin{bmatrix}
		\Lambda_{1} & \\
		& \Lambda_{2}
		\end{bmatrix}
		\begin{bmatrix}
		Q_{1}^{T} \\
		Q_{2}^{T}
		\end{bmatrix},
	\end{equation}
	where 
	$\Lambda_1 \in\mathbb R^{k\times k}$
	and $\Lambda_2 \in\mathbb R^{(n-k)\times (n-k)}$
	contain eigenvalues,
	and $Q_1 \in \mathbb{R}^{n \times k}$ and $Q_2 \in \mathbb{R}^{n \times (n-k)}$ are orthonormal eigenvectors. 
	For a given shift  $\sigma\in\mathbb R$, define the  shifted matrix
	\begin{equation*}
		A_{\sigma} = A + \sigma Q_{1}Q_{1}^{T}. 
	\end{equation*}
	Then the eigenvalue decomposition of $A_{\sigma}$ has the following form
	\begin{equation}\label{eq:eigshift}
	 A_{\sigma}  = \begin{bmatrix}
	  Q_{1} & Q_{2}
	\end{bmatrix}
	\begin{bmatrix}
	\Lambda_{1} + \sigma I & \\
	& \Lambda_{2}
	\end{bmatrix}
	\begin{bmatrix}
	Q_{1}^{T} \\ Q_{2}^{T}
	\end{bmatrix}.
	\end{equation}
\end{proposition}
\begin{proof}
    See Appendix~\ref{appx:prop3-2}.
\end{proof}

Suppose we are interested in the eigenvalues $\Lambda_2$ and 
the corresponding eigenvectors $Q_2$.
By Proposition~\ref{prop:hotelling},  
if we choose the shift $\sigma$ sufficiently large,
eigenvalues in $\Lambda_2$ will always correspond to the 
$n-k$ smallest eigenvalues of $A_\sigma$,
since the unwanted eigenvalues $\Lambda_1$ are shifted away to $\Lambda_1+\sigma I$.
The corresponding eigenvectors remain unchanged.
This is exactly what we need to untangle the 
unwanted $h-1$ zero eigenvalues of $L^{\rm p}_{\rm n}$ in~\eqref{eq:blockAp}
from the rest of the eigenvalues of $\lambda_i$. 

Recall the eigenvalue decomposition of $L^{\rm p}_{\rm n}$ in~\eqref{eq:blockAp}.
To shift away the unwanted $h-1$ zero eigenvalues, 
we can apply Hotelling's deflation with a shift $\sigma$ to obtain
\begin{equation}\label{eq:lsigma}
L^{\sigma}_{\rm n} := L^{\rm p}_{\rm n} + \sigma U_{2}U_{2}^{T},
\end{equation}
where recall that $U_2$ is an orthonormal basis for the range of $C$.
If the shift $\sigma$ is chosen such that 
$\sigma > \lambda_{k}$, where 
$\lambda_k$ is the $k$-th smallest eigenvalue in $\Lambda_{\rm v}$,
then our desired $k$ smallest eigenvalues in $\Lambda_{\rm v}$
are corresponding to the $k$ smallest eigenvalues of 
$L^{\sigma}_{\rm n}$.
Consequently, the eigenspace mixing issue is solved. On the other hand, by~\eqref{eq:pmat2}, the shifted matrix in~\eqref{eq:lsigma} can be expressed 
as follows
\begin{equation}\label{eq:asigma}
L^{\sigma}_{\rm n} = PL_{\rm n}P + \sigma (I-P) 
= P(L_{\rm n}-\sigma I)P + \sigma I.
\end{equation}

Then the matrix-vector multiplication with $L^{\sigma}_{\rm n}$
only requires operations with $P$ and $L_{\rm n}$.
This is extremely beneficial for large-scale fair SC models.

\subsection{Algorithm and implementation}
As we described in the previous section, Hotelling's deflation with a proper choice of $\sigma$
resolves the eigenspace mixing issue for 
the projected eigenvalue problem~\eqref{eq:eig2}.
To summarize, the solution of the constrained trace minimization~\eqref{eq:normal-prob-fair}
can now be characterized by the following 
proposition.

\begin{proposition}
Let $L^{\sigma}_{\rm n} \in \mathbb R^{n\times n}$ be defined
as in~\eqref{eq:asigma} and  assume $\sigma$ is sufficiently large such that
$\sigma > \lambda_{k}(L^{\rm v}_{\rm n})$,
where $\lambda_k(L^{\rm v}_{\rm n})$ is the 
$k$-th smallest eigenvalue of 
$L^{\rm v}_{\rm n}$ in~\eqref{eq:lveig}. 
Then $H$ is a solution to the trace
minimization~\eqref{eq:normal-prob-fair}
if and only if 
$H = D^{-\frac{1}{2}}X$,
where $X=[x_1,x_2,\dots,x_k]\in\mathbb R^{n\times k}$
contains the $k$ eigenvectors corresponding to the $k$ 
smallest eigenvalues of $L^{\sigma}_{\rm n}$.
\end{proposition}

The final algorithm
based on projected eigenproblem~\eqref{eq:eig2} and Hotelling's deflation is presented in Algorithm~\ref{code:algo3}, called 
scalable FairSC, s-FairSC in short.  

\begin{algorithm}[h]
    \caption{Scalable FairSC (s-FairSC)}
    \label{code:algo3}
    \begin{algorithmic}[1]
            \Input weighted adjacency matrix $W \in \mathbb{R}^{n \times n}$;
		degree matrix $D \in \mathbb{R}^{n \times n}$;
    group-membership matrix $F \in \mathbb{R}^{n \times (h-1)}$;
    shift $\sigma \in \mathbb{R}$;
    $k \in \mathbb{N}$
    \Output a clustering of indices $1:n$ into $k$ clusters
    
    \State compute the Laplacian matrix $L = D - W$;
    
    \State set $L_{\rm n} = D^{-\frac{1}{2}}LD^{-\frac{1}{2}}$,
    and $C = D^{-\frac{1}{2}}F$;
    
    \State compute the $k$ smallest eigenvalues of $L^{\sigma}_{\rm n}$ in~\eqref{eq:asigma} and the corresponding eigenvectors as columns of $X \in \mathbb{R}^{n \times k}$;
    
    \State apply $k$-means clustering to the rows of $H = D^{-\frac{1}{2}}X$.
    
    \end{algorithmic}
\end{algorithm}

\paragraph{Implementation issues.}~\label{sec:impl-iss}
A few implementation issues regarding 
s-FairSC (Algorithm~\ref{code:algo3}) are in order. (i) For computing eigenvalue of $L^{\sigma}_{\rm n}$,
		we can use an iterative eigensolver, such as {\tt eigs} in MATLAB, which is based on ARPACK~\citep{lehoucq1998arpack}, 
        an implicitly restarted Arnoldi method.
        The eigensolver only needs to access
		$L^{\sigma}_{\rm n}$ through the matrix-vector multiplication
		$L^{\sigma}_{\rm n}w$.
		By~\eqref{eq:asigma}, 
		\begin{align*}
		    L^{\sigma}_{\rm n} w 
		% = \Big(P(L_{\rm n}-\sigma I)P + \sigma I \Big) w
		%&= P(L_{\rm n}-\sigma I)Pw + \sigma w \\
		&= P(L_{\rm n}(Pw)) - \sigma Pw + \sigma w.
		\end{align*}
		%It requires to call two times the projection operation with
		%$P$.  
(ii) The projection
$P$ in~\eqref{eq:pmat2}
can be written as $P = I - C(C^TC)^{-1}C^T$, see~\cite{Golub:2000}.
Consequently, 
\begin{equation} \label{eq:mv-Pw}
Pw =  (I - C(C^TC)^{-1}C^T) w =  w - Cz,
\end{equation}
where $z$
%$ z=(C^{T} C)^{-1} C^{T} w$ 
    is the solution to the least-squares problem
    \[
\min_{z} \|Cz - w\|_2. 
\]
	For small to moderate size problems, direct LS solver 
    can be applied to computing $z$.
    For large scale problems, iterative methods such as LSQR~\citep{paige1982lsqr}
	can be applied; this is in line with the inner-outer iteration methods for eigenvalue computation, see e.g., \cite{Golub:2000}.
(iii) For an appropriate choice of shift $\sigma$, one can use an
estimation for the largest eigenvalue of $L_{\rm n}$.
%as recommended
%in~\cite{Lin:2021}.
Such a shift also guarantees the numerical stability of Hotelling's deflation~\citep{Lin:2021}.

\paragraph{Time Complexity.} \label{sec:complexity}
The complexity of s-FairSC is dominated 
by computing $k$ eigenpairs of the matrix $L_{\rm n}^{\sigma}$. 
To use a modern Krylov subspace eigensolver, 
say the function {\tt eigs} in MATLAB, the two leading costs are 
(1) the matrix-vector product $L_{\rm n}^{\sigma}w$, and (2) the orthonormalization of
basis vectors of Krylov subspace eigensolver. 
For (1), the complexity is $\mathcal{O}(m+nh^2)$ and for (2), it is $\mathcal{O}(nk^2)$, 
where $n = |V|$, $m$ is the number of non-zero elements of $W$ of graph $\mathcal{G}(V,W)$, $h$ is the number of groups
and assume that the product $Pw$ for the projection matrix $P$
is computed by a direct least squares solver since $h$ is 
typically small. Therefore, the complexity of s-FairSC 
is $\mathcal{O}(m+n(h^2+ k^2))$, where the constant of $\mathcal{O}(\cdot)$
depends on the number of restarts of subspace iterations, 
usually about 10 to 20. 
Using
the same analysis, the complexity of SC (without fairness constraints) is $\mathcal{O}(m+nk^2)$.
Since $h$ is typically small, say $h = 10$, it explains observations that
s-FairSC is as fast as SC; see numerical results in 
Section~\ref{sec:ex-results}.
%(Figure~\ref{fig:benchmarkSBM} and Figure~\ref{fig:time_lapla}).

\subsection{Related work} 
The idea of transforming the optimization problem~\eqref{eq:normal-prob-fair} to an equivalent eigenvalue
problem is very natural. 
For the case of $k=1$, the projected eigenvalue problem~\eqref{eq:eig2}
was considered in~\cite{Golub:1973}
and~\cite[p.~621]{golub1996matrix}.

The projected eigenvalue problem~\eqref{eq:eig2}
is a form of so-called {\it constrained eigenvalue problems},
which is more generally formulated as 
$Ax = \lambda M x$ subject to $C^Tx = 0$,
where $A$ and $M$ are symmetric and $M$ is positive definite.
The constrained eigenvalue problems 
are found in many applications. There are a number of approaches available;
see~\cite{doi:10.1137/S0895479800381331}
for an algorithm for positive semidefinite $A$ with a known nullspace;
~\cite{baker2009preconditioning} for a preconditioning technique;
~\cite{Golub:2000} for a Lanczos process with inner-outer iterations 
to handle large matrices;
and ~\cite{porcelli2015solution} for a solution procedure 
within the structural finite-element code NOSA-ITACA.
For constrained eigenvalue problems,
the matrix $C$ is typically corresponding to the nullspace of $A$, 
and the constraint $C^Tx = 0$ is to avoid computing ``null eigenvectors''. 
Since the entire nullspace of $A$ is avoided,
there is no eigenvector selection issue as in our problem~\eqref{eq:eig2}.
\cite{simoncini2003algebraic} proposed a reformulation of
the constrained eigenproblem based on null eigenvalue shifting.
Her approach essentially includes Hotelling's deflation 
as a special case, but with a goal to shift away the entire 
nullspace of $A$.  By discussion in Section~\ref{sec:algs}, we show that 
Hotelling's deflation is also capable of splitting the unwanted 
null vectors from those desired ones.

\section{EXPERIMENTS}\label{sec:expr}
In this section, we present experimental results on the proposed s-FairSC 
(Algorithm \ref{code:algo3}).
Similar to SC and FairSC, s-FairSC is implemented in MATLAB\textsuperscript{\textregistered}.\footnote{SC and FairSC code: \url{https://github.com/matthklein/fair_spectral_clustering}.
s-FairSC code:  \url{https://github.com/jiiwang/scalable_fair_spectral_clustering}}
The results are obtained from
a MacBook Pro with an 8-core i9 processor @2.3 GHz, 
16 GB memory, and 16 MB L3 cache.

\subsection{Datasets} \label{sec:ex-data}

\paragraph{Modified Stochastic block model (m-SBM).} \label{sec:sbm}
The stochastic block model (SBM)~\citep{HOLLAND1983109} is 
a random graph model with planted blocks (ground-truth clustering).
It is widely used to generate synthetic networks 
for clustering and community 
detection~\citep{rohe2011spectral,balakrishnan2011noise,lei2015consistency,sarkar2015role}.
To take group fairness into account, 
we use a modified SBM (m-SBM) proposed 
by~\cite{kleindessner2019guarantees}
to generate the test graph $\mathcal{G}(V,W)$.
%In m-SBM, $n$ vertices are assigned to $k$
%clusters for a prescribed fair ground-truth clustering
%$V = C_1 \cup \cdots \cup C_k$, 
%and edges are placed between  vertex pairs with probabilities dependent only on 
%whether they belong to the same cluster or group, 
%hence graph $\mathcal{G}(V,W)$ is generated (see Appendix~\ref{appx:msbm} for details).
In m-SBM, $n$ vertices are assigned to $k$
prescribed (ground-truth) clusters $V = C_1 \cup \cdots \cup C_k$, 
and between any pair of vertices, an edge is placed with 
a probability that depends only on the clusters of the two vertices
(see Appendix~\ref{appx:msbm} for details).
Let $V = \widehat{C}_1 \cup \cdots \cup \widehat{C}_k$ be
a computed clustering. 
%The quality with respect  to the fair ground-truth is measured by the {\it error rate} on the proportion of misclustered vertices:
The discrepancy between the computed and ground-truth clustering is measured by the {\it error rate} 
of clustering  (proportion of misclustered vertices):
\begin{equation} \label{eq:err}
\mbox{Err}(\widehat{H} - H) := 
\frac{1}{n} \min_{J \in \Pi_{k}}\norm{\widehat{H}J - H}^{2}_{F},
\end{equation}
where $H$ and $\widehat{H}$ are the ground-truth and computed cluster indicator matrices,  respectively, 
and $\Pi_k$ is the set of all possible $k \times k$ permutation matrices. 

\paragraph{FacebookNet.} \label{sec:facebook}
FacebookNet\footnote{\url{http://www.sociopatterns.org/datasets/high-school-contact-and-friendship-networks/}} 
is a dataset that  collects Facebook friendship relations 
between students in a high school in France in 2013.
This social network dataset was studied for information propagation and opinion formation~\cite{10.1371/journal.pone.0136497} and for clustering~\citep{crawford2018cluenet,kleindessner2019guarantees,chodrow2021generative}.
In graph $\mathcal{G}(V, W)$, $V$ is the set of students ($n = |V| = 155$),
and an edge %encoded in $W$ 
represents a friendship between two students.
Students are divided by gender into two 
groups $V = V_1 \cup V_2$, with $|V_1| = 70$ of girls and  $|V_2| = 85$ of boys.

\paragraph{LastFMNet.} \label{sec:lastfm}
LastFMNet\footnote{\url{http://snap.stanford.edu/data/feather-lastfm-social.html}}~\citep{feather}
is a real-world dataset that contains mutual follower relations among 
users of Last.fm,  a recommender-system-based online radio and music community in Asia. 
LastFMNet was collected from public API
in 2020 and used to study the distribution of vertex features on graphs.
In graph $\mathcal{G}(V, W)$, 
$V$ is the set of users with $n = |V| = 5576$, 
and an edge %encoded in $W$ 
represents a mutual follower friendship between two users. %In addition, 
LastFMNet also records nationalities of the users $V = V_1 \cup \dots \cup V_6$ with $|V_1| = 1073$, $|V_2| = 505$, $|V_3| = 645$,  $|V_4| = 1266$, $|V_5| = 558$ and $|V_6| = 1529$.
$\mathcal{G}(V, W)$ has 19587 edges, and the density is 0.00013.

\paragraph{Random Laplacian.} \label{sec:random}
To create a random Laplacian of graph $\mathcal{G}(V, W)$, 
% and assume every edge of $E$ has an equal weight of 1.
we first generate a random symmetric weight matrix $W \in \mathbb{R}^{n \times n}$ with 
prescribed sparsity $s$ and $n=|V|$,
and then set the degree matrix $D = \mbox{diag}(W {\bf 1}_n)$
and the Laplacian $L = D - W$. 
The matrix $F \in \mathbb{R}^{n \times (h-1)}$ in the constraints of the fair SC model~\eqref{eq:normal-prob-fair} 
is also constructed as a random matrix. 
%In this case, $F$ does not contain group-membership information, but only preserves the dimension. 
Here, $F$ is not for the group-membership information but only acts as a placeholder.
% The MATLAB script to generate the random Laplacian is given below:
% \begin{verbatim}
%     A = sprand(n,n,s);  % s is for sparsity, say s = 0.1
%     A = tril(A,-1);
%     A = (A+A')/2;
%     A = A.*~eye(size(A));
%     D = diag(A*ones(n,1));
%     F = rand(n,h-1);
% \end{verbatim}
We will use this dataset to show the scalability of algorithms.

\subsection{Experimental results} \label{sec:ex-results}

\paragraph{Experiment 1.} \label{exp:1}
This experiment is conducted on the m-SBM
to compare the error rate~\eqref{eq:err} and running time of SC,  FairSC, and s-FairSC. 
%Figure~\ref{fig:benchmarkSBM} depicts 
%the error rate and running timing of SC, FairSC and s-FairSC.
Figure~\ref{fig:benchmarkSBM} depicts the computation results.
SC and s-FairSC are tested for model sizes from $n = 1000$ to $10000$. 
FairSC stops at $n = 4000$ due to its high computational cost, 
echoing results reported in
% Similar to the experiment reported in 
~\cite{kleindessner2019guarantees}.

\begin{figure}[ht!]
\begin{center}
    \includegraphics[width=0.48\textwidth]{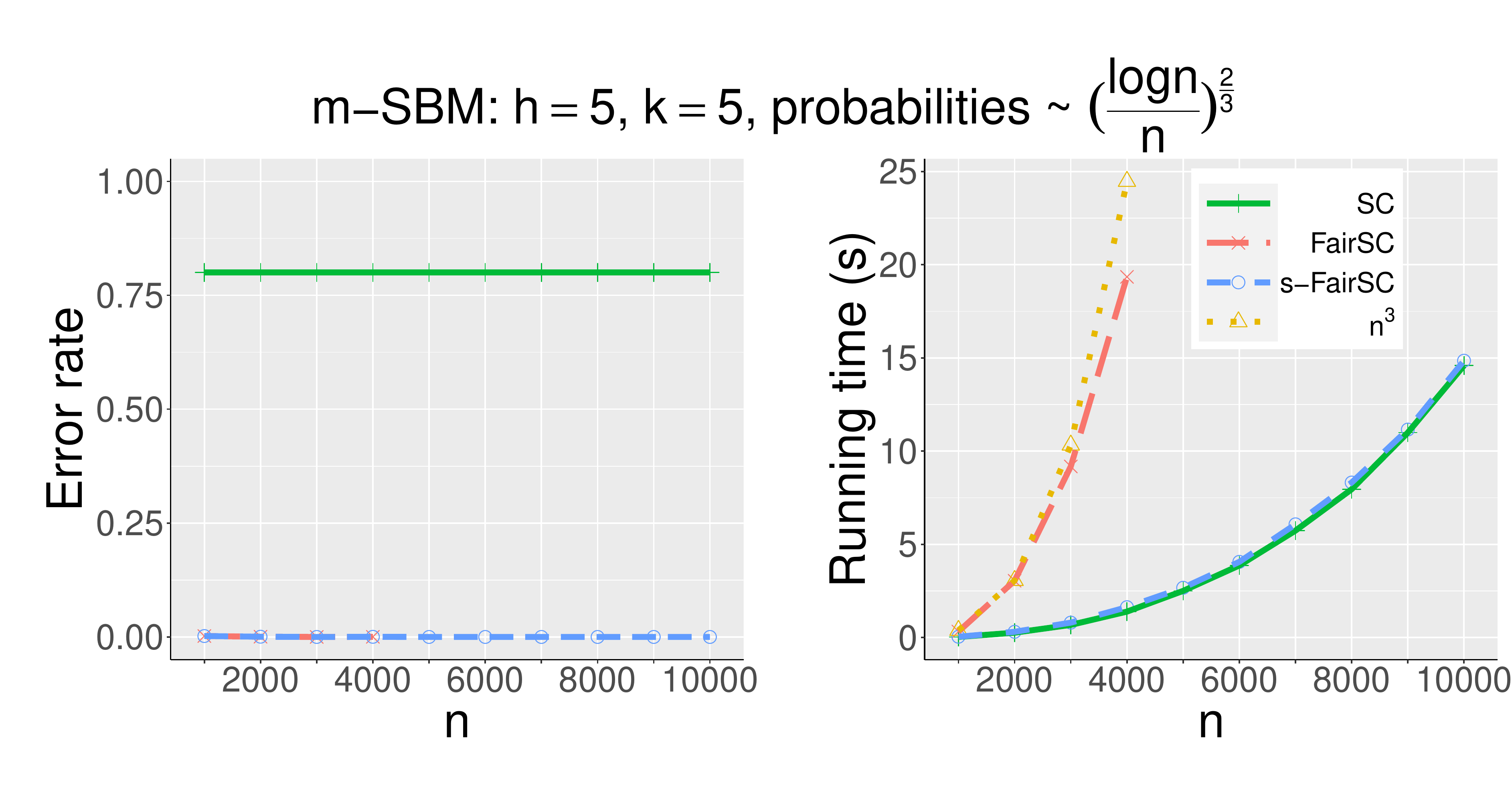} 
    \end{center}
\caption{Error rate and running time (in seconds) of SC, FairSC and s-FairSC of
an m-SBM with $h = 5$, $k = 5$, and 
edge connectivity probabilities proportional 
to $(\frac{\log n}{n})^\frac{2}{3}$.
% : $a = 10\times (\frac{\log n}{n})^{2/3}, 
% b = 7\times (\frac{\log n}{n})^{2/3}, 
% c = 4\times (\frac{\log n}{n})^{2/3}, 
% d = (\frac{\log n}{n})^{2/3}$. 
}
\label{fig:benchmarkSBM}
\end{figure}
    
%From Figure~\ref{fig:benchmarkSBM}, we observe that
%SC fails to recover 
%the fair ground-truth clustering, 
%while both FairSC  
%and s-FairSC are able to retrieve the fair ground-truth clustering. 

From Figure~\ref{fig:benchmarkSBM}, we observe that
both  FairSC  and s-FairSC successfully retrieve 
the fair ground-truth clustering,  but SC fails. 
We can see that  s-FairSC is as good as FairSC  in terms of error rate for the computed clustering.
But its running time is only a fraction of that of FairSC;
e.g., for $n = 4000$, s-FairSC is 12$\times$ faster than FairSC. 
We also observe that s-FairSC is as scalable as SC, but the latter does not account for the fairness constraints. 

\paragraph{Experiment 2.}
We use the FacebookNet dataset 
%to quantify the approximation to the exact group fairness.
to quantify the group fairness in the computed clustering.
Table~\ref{tab:checklemma} records the quantities  from Definition~\ref{def:fair}. 

\begin{table}[ht!]
    \begin{center}
    \resizebox{0.7\columnwidth}{!}{
        \begin{tabular}{|| c || c | c | c||} 
         \hline
         & {\bf SC}  & {\bf FairSC}  & {\bf s-FairSC}  \\ 
         \hline\hline 
         
         & \multicolumn{3}{c||}{} \\ [-1em]
         $\frac{|V_{1}|}{|V|}$ & \multicolumn{3}{c||}{0.4516} \\ [0.2em]
         \hline 
         & & & \\ [-1em]
         $\frac{|V_{1}\cap \widehat{C}_{1}|}{|\widehat{C}_{1}|}$  & 0.6528 & 0.3537 & 0.3537  \\ [0.5em]
         \hline
         
          & & & \\ [-1em]
         $\frac{|V_{1}\cap \widehat{C}_{2}|}{|\widehat{C}_{2}|}$ & 0.2771 &  0.5616 & 0.5616 \\ [0.5em]
         \hline 
         
          & \multicolumn{3}{c||}{} \\ [-1em]
         $\frac{|V_{2}|}{|V|}$ & \multicolumn{3}{c||}{0.5484} \\ [0.2em]
         \hline 
         & & & \\ [-1em]
         $\frac{|V_{2}\cap \widehat{C}_{1}|}{|\widehat{C}_{1}|}$  & 0.3472 & 0.6463 & 0.6463 \\ [0.5em]
         \hline
         & & & \\ [-1em]
         $\frac{|V_{2}\cap \widehat{C}_{2}|}{|\widehat{C}_{2}|}$ & 0.7229 & 0.4384  & 0.4384 \\ [0.5em]
         \hline
        \end{tabular}
        }
    \end{center}
    \caption{Fractions of group membership within each cluster for the recovered clustering $V = \widehat{C}_1\cup\widehat{C}_2$.
    % $ = V_1 \cup V_2.$
    }\label{tab:checklemma}
\end{table}

%The concept of the average balance 
%introduced in~\cite{chierichetti2017fair} 
%has been used as a metric to assess the fairness of a  clustering.

The average balance  introduced in~\cite{chierichetti2017fair} 
has been used to measure fairness in clustering.
Given a clustering $V = C_{1}\cup\cdots\cup C_{k}$ and group partition $V = V_1 \cup \cdots \cup V_h$, 
the balance of cluster $C_{\ell}$ for $\ell = 1, 2, \cdots, k$
is defined as  
\begin{equation} \label{eq:bal}
    \mbox{balance}(C_{\ell}) := \min_{s \neq s' \in \{1, \cdots, h\}} \frac{|V_{s}\cap C_{\ell}|}{|V_{s'}\cap C_{\ell}|} \in [0,1].
\end{equation}
The average balance is then given by
\begin{equation} \label{eq:avebal}
    \mbox{Average\_Balance} := \frac{1}{k}\sum_{l=1}^{k} \text{balance}(C_{\ell}).
\end{equation}
A higher balance implies a fairer clustering; see Appendix~\ref{appx:bal} for an explanation. 

\begin{figure}[ht!]
\centering{
\includegraphics[width=0.3\textwidth]{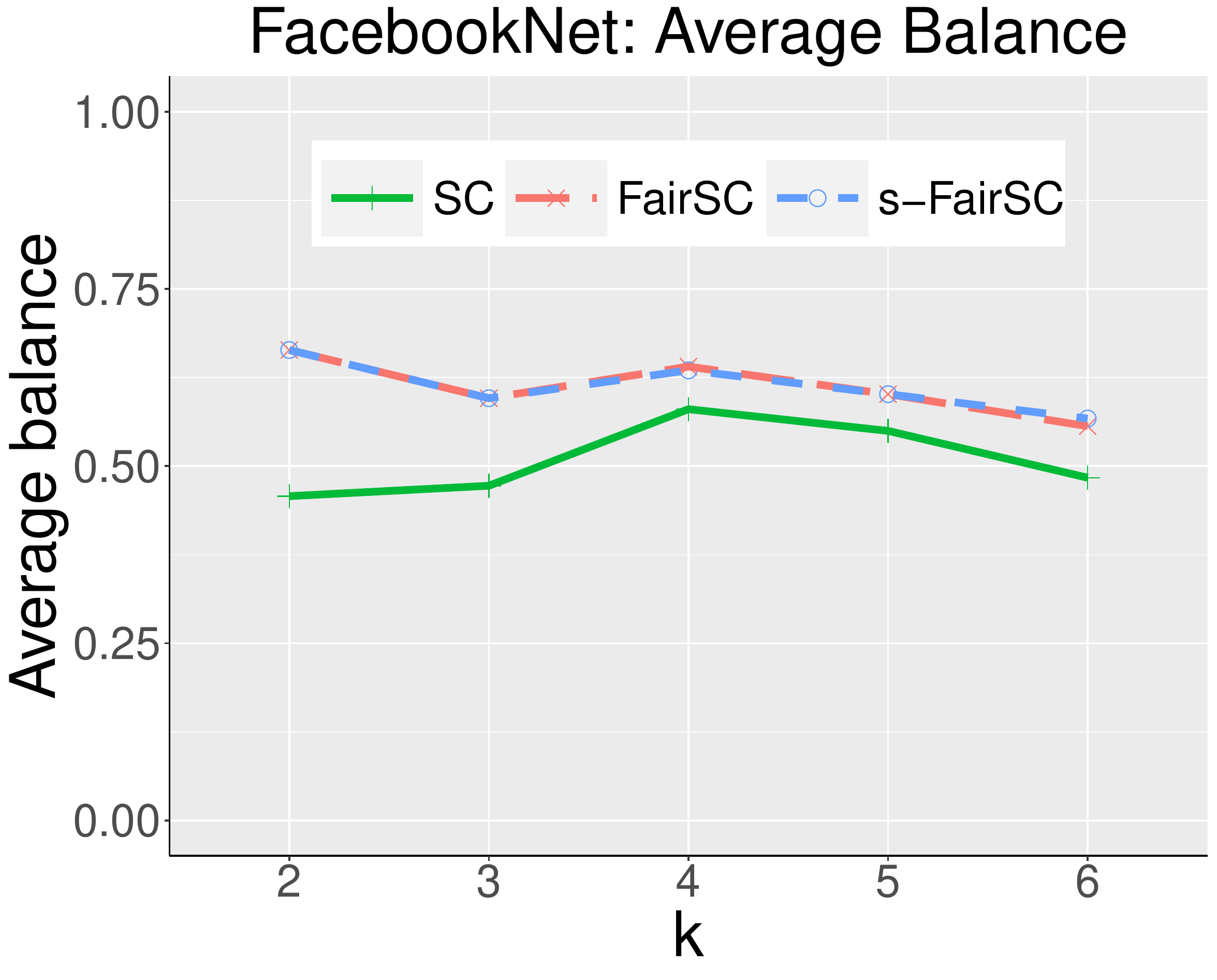} }
\caption{
Average\_Balance of SC, FairSC, and s-FairSC on FacebookNet
as a function of the number $k$ of clusters.
}
\label{fig:fb-bal}
\end{figure}

By Table~\ref{tab:checklemma} and Figure~\ref{fig:fb-bal},
we observe that since problem~\eqref{eq:normal-prob} is relaxed to 
problem~\eqref{eq:normal-prob1}, the equality in~\eqref{eq:fair} 
does not hold.
Nevertheless, 
both FairSC and s-FairSC have improved fairness compared to SC.  
In addition, FairSC and s-FairSC produce almost identical results. 

\paragraph{Experiment 3.} \label{exp:3}
In this experiment, we use LastFMNet to compare the running time of SC, FairSC, and s-FairSC. 
We also measure average balance~\eqref{eq:avebal} to evaluate the
fairness of clustering by the algorithms. %SC, FairSC, and s-FairSC. 
The running time %of SC, FairSC, and s-FairSC 
as a function of the number $k$ of clusters %on LastFMNet dataset
is illustrated in Figure~\ref{fig:time_fm}. 
We observe that when $k \geq 5$,
% the running time of the three algorithms stabilize: 
s-FairSC is 7$\times$ faster than FairSC, 
and it is as fast as SC.
Figure~\ref{fig:fm-bal} shows 
the values of
Average\_Balance as a function of the number $k$ of clusters. 
Both FairSC and s-FairSC have higher values of Average\_Balance than SC, 
indicating they %that FairSC and s-FairSC 
have improved fairness compared to SC. 

%Additionally, FairSC and s-FairSC produce almost identical results.

\begin{figure}[ht!]
\begin{center} 
\includegraphics[width=0.32\textwidth]{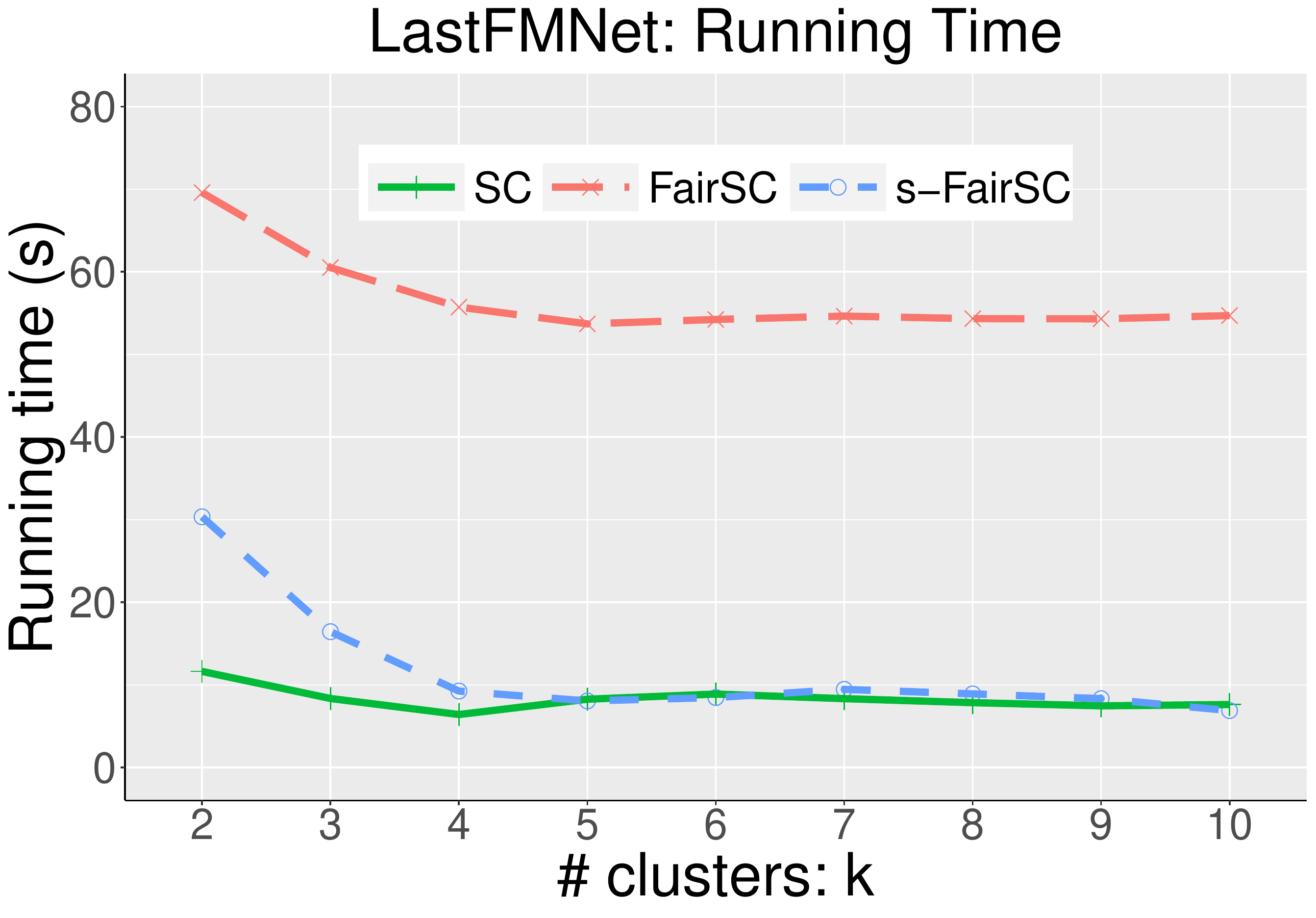}
\end{center} 
\caption{Running time (in seconds) of SC, FairSC, and s-FairSC on LastFMNet as a function of 
the number $k$ of clusters.
}
\label{fig:time_fm}
\end{figure} 

\begin{figure}[ht!]
\begin{center} 
\includegraphics[width=0.32\textwidth]{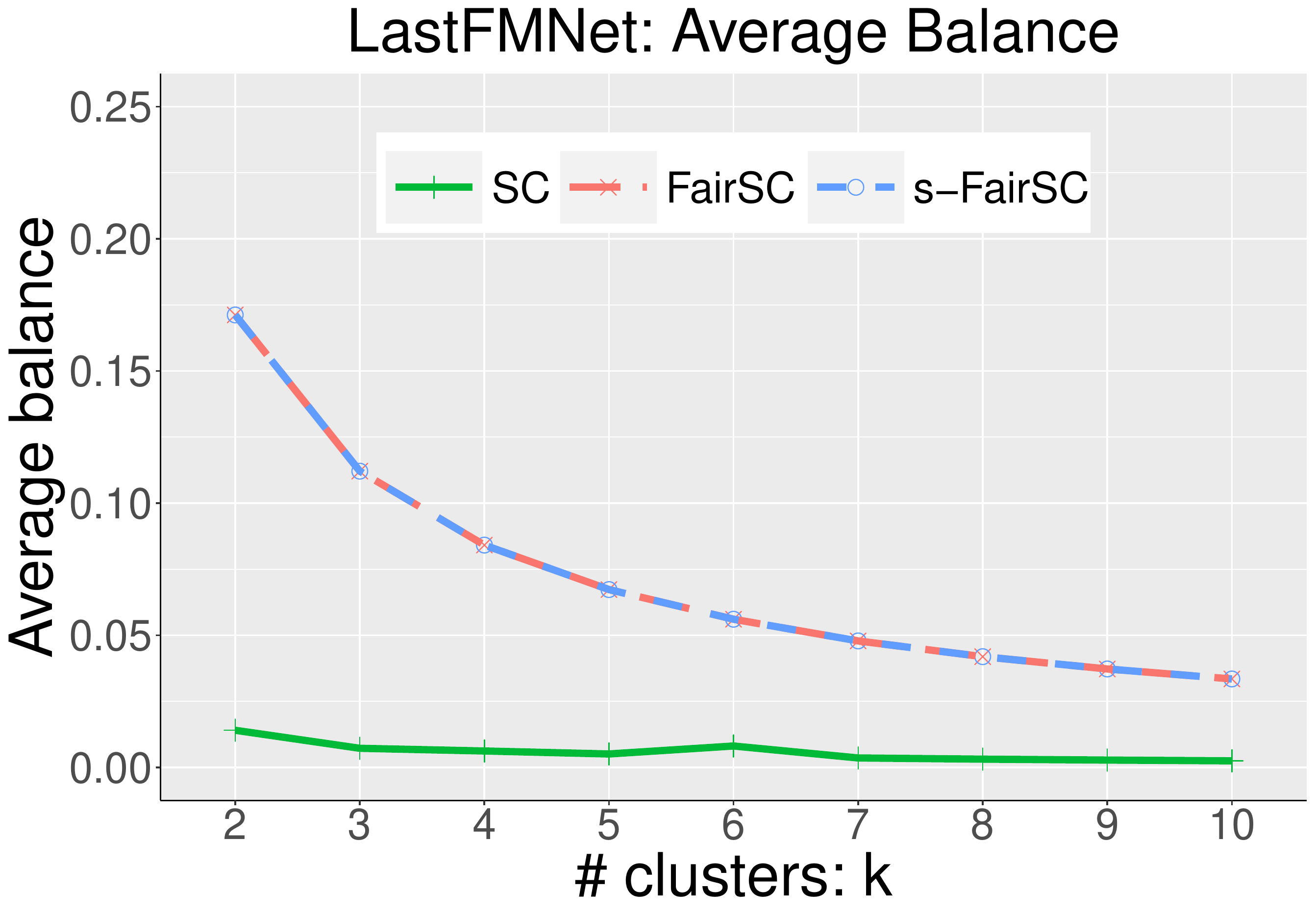} 
\end{center} 
\caption{
Average\_Balance of SC, FairSC, and s-FairSC 
on LastFMNet
as a function of 
the number $k$ of clusters.
}
\label{fig:fm-bal}
\end{figure}

\paragraph{Experiment 4.} \label{exp:4}
In this experiment, we use random Laplacian to demonstrate that 
s-FairSC has a similar scalability as SC.
%from the perspective of
%numerical linear algebra in solving the trace minimization 
%problems~\eqref{eq:normal-prob-fair} and ~\eqref{eq:normal-prob2} with or without linear constraints.
Figure~\ref{fig:time_lapla} reports the running time of SC (Algorithm~\ref{code:algo1}) for solving the SC model \eqref{eq:normal-prob1}
and the s-FairSC  
(Algorithm~\ref{code:algo3}) for solving the fair SC model~\eqref{eq:normal-prob-fair}. 
Here, the model sizes range from $5000$ to $10000$ with the number of groups $h =5$ and
different numbers of clusters $k = 5, 8, 10$. We observe that s-FairSC is only slightly 
more expensive than SC and is as scalable as SC. 

\begin{figure}[ht!]
\begin{center} 
\includegraphics[width=0.36\textwidth]{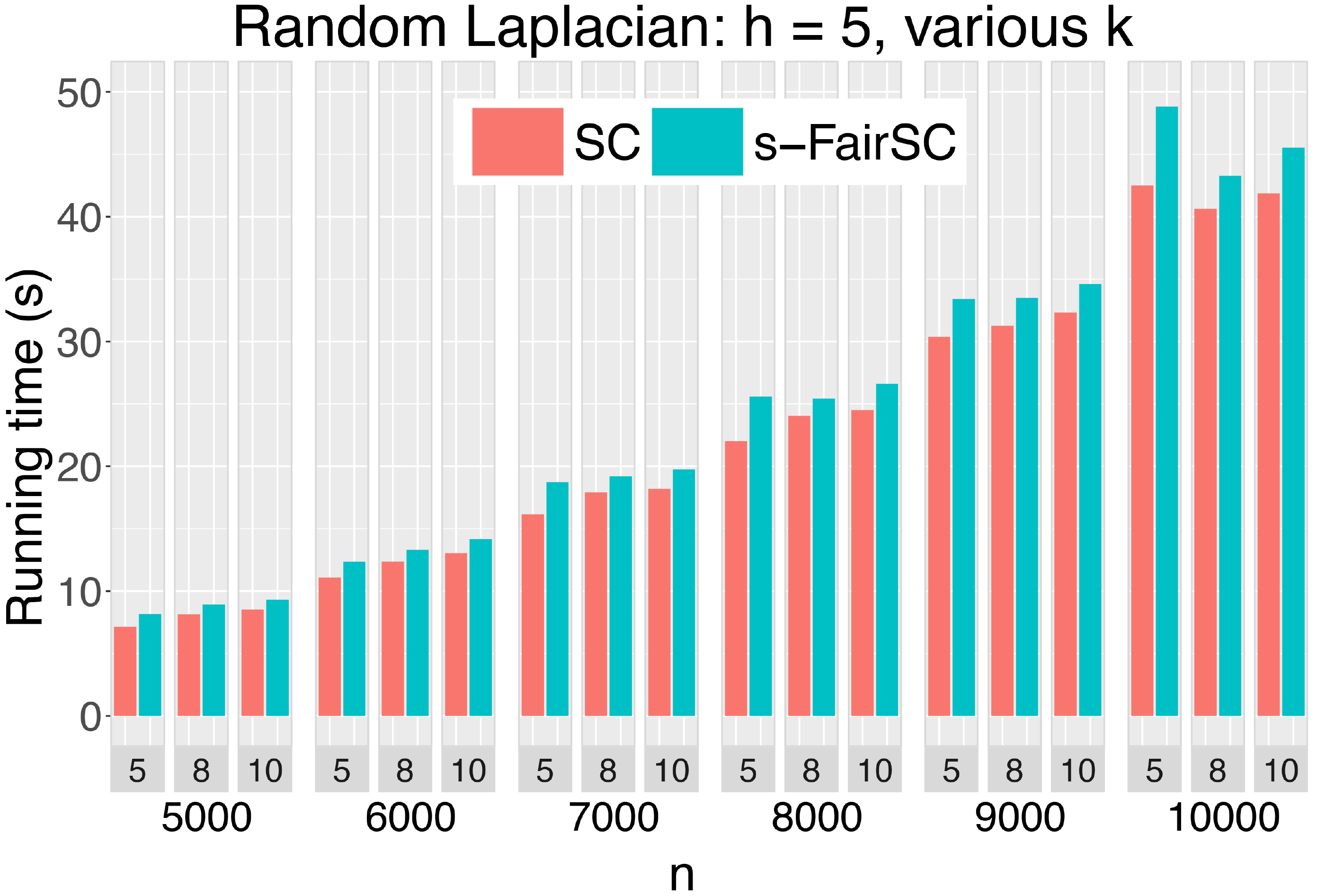}
\end{center} 
\caption{Running time (in seconds) of SC and s-FairSC on random Laplacian
with $h = 5$ and $k \in \{5, 8, 10\}$. 
}
\label{fig:time_lapla}
\end{figure} 

\section{CONCLUDING REMARKS} \label{sec:concl}

FairSC (Algorithm~\ref{code:algo2})
is able to recover fairer clustering, but sacrifices the performance and scalability.
In this paper, we presented a scalable FairSC (s-FairSC, Algorithm~\ref{code:algo2}) by incorporating nullspace projection and Hotelling's deflation. All computational kernels of s-FairSC only involve the sparse 
matrix-vector products, so the algorithm can fully exploit the sparsity 
of the fair SC model~\eqref{eq:normal-prob-fair} and is scalable in the sense that it only 
has a marginal increase in computational costs compared to SC without fairness constraints.

We note that the non-overlap of groups leads to the full rank
of the group indicator matrix $H$, which simplifies the rest of
presentation substantially. 
An interesting extension to the current work is to incorporate group overlapping. 
The overlap of groups may lead to the rank deficiency of $F$. 
In this case, a simple approach is to perform rank-revealing factorization of $F$ 
first, and the rest of the discussion will hold. However, 
it is a subject of further study on how to avoid the rank-revealing factorization, 
and maintain the sparsity of $F$ in computation. Another intriguing problem is the development of scalable algorithms to solve 
the group fairness condition~\eqref{eq:fair} in a less stringent manner. For instance, 
the following notion of group fairness clustering with lower and upper bounds is introduced in~\cite{NEURIPS2019_fc192b0c}:
\begin{equation} \label{eq:bera} 
 \beta_{s} \leq \frac{|V_{s}\cap C_{\ell}|}{|C_{\ell}|} \leq \alpha_{s}
\quad \mbox{for $s = 1, 2, \cdots, h$},
\end{equation}
where $\beta_{s}$ and $\alpha_{s}$ are lower and upper 
bounds for group $V_s$, respectively, and
$0 < \beta_{s} \leq \alpha_{s} < 1$.
A further topic is to extend the s-FairSC  
to individual fairness, where any two individuals who are similar with respect 
to a specific sensitive attribute should be treated 
similarly~\citep{dwork2012fairness,zemel2013learning}. 
An SC model with
individual fairness constraints is devised in \cite{gupta2022consistency}. 
However, the existing algorithm~\citep{gupta2022consistency} is not scalable due to the high costs of its computational kernels.

\subsubsection*{Acknowledgements}
Wang and Bai was supported in part
by NSF grant 1913364. Lu was supported by NSF grant 2110731. Davidson was supported in part by NSF grant 1910306, and a gift from Google.
We would like to thank anonymous reviewers for their constructive comments that have significantly improved the presentation.

\bibliography{ref}

\begin{thebibliography}{52}
\providecommand{\natexlab}[1]{#1}
\providecommand{\url}[1]{\texttt{#1}}
\expandafter\ifx\csname urlstyle\endcsname\relax
  \providecommand{\doi}[1]{doi: #1}\else
  \providecommand{\doi}{doi: \begingroup \urlstyle{rm}\Url}\fi

\bibitem[Agarwal et~al.(2019)Agarwal, Dud{\'\i}k, and Wu]{agarwal2019fair}
A.~Agarwal, M.~Dud{\'\i}k, and Z.~S. Wu.
\newblock Fair regression: Quantitative definitions and reduction-based
  algorithms.
\newblock In \emph{International Conference on Machine Learning}, pages
  120--129. PMLR, 2019.

\bibitem[Aghaei et~al.(2019)Aghaei, Azizi, and Vayanos]{aghaei2019learning}
S.~Aghaei, M.~J. Azizi, and P.~Vayanos.
\newblock Learning optimal and fair decision trees for non-discriminative
  decision-making.
\newblock In \emph{Proceedings of the Thirty-Third AAAI Conference on
  Artificial Intelligence and Thirty-First Innovative Applications of
  Artificial Intelligence Conference and Ninth AAAI Symposium on Educational
  Advances in Artificial Intelligence}, pages 1418--1426, 2019.

\bibitem[Amini et~al.(2019)Amini, Soleimany, Schwarting, Bhatia, and
  Rus]{amini2019uncovering}
A.~Amini, A.~P. Soleimany, W.~Schwarting, S.~N. Bhatia, and D.~Rus.
\newblock Uncovering and mitigating algorithmic bias through learned latent
  structure.
\newblock In \emph{Proceedings of the 2019 AAAI/ACM Conference on AI, Ethics,
  and Society}, pages 289--295, 2019.

\bibitem[Arbenz and Drmac(2002)]{doi:10.1137/S0895479800381331}
P.~Arbenz and Z.~Drmac.
\newblock On positive semidefinite matrices with known null space.
\newblock \emph{SIAM Journal on Matrix Analysis and Applications}, 24\penalty0
  (1):\penalty0 132--149, 2002.
\newblock \doi{10.1137/S0895479800381331}.
\newblock URL \url{https://doi.org/10.1137/S0895479800381331}.

\bibitem[Bach and Jordan(2003)]{bach2003learning}
F.~Bach and M.~Jordan.
\newblock Learning spectral clustering.
\newblock \emph{Advances in neural information processing systems}, 16, 2003.

\bibitem[Bach and Jordan(2006)]{bach2006learning}
F.~R. Bach and M.~I. Jordan.
\newblock Learning spectral clustering, with application to speech separation.
\newblock \emph{The Journal of Machine Learning Research}, 7:\penalty0
  1963--2001, 2006.

\bibitem[Bai et~al.(2000)Bai, Demmel, Dongarra, Ruhe, and van~der Vorst~{\rm
  (editors)}]{bddrv:2000}
Z.~Bai, J.~Demmel, J.~Dongarra, A.~Ruhe, and H.~van~der Vorst~{\rm (editors)}.
\newblock \emph{Templates for the solution of Algebraic Eigenvalue Problems: A
  Practical Guide}.
\newblock SIAM, Philadelphia, 2000.

\bibitem[Baker and Lehoucq(2009)]{baker2009preconditioning}
C.~G. Baker and R.~B. Lehoucq.
\newblock Preconditioning constrained eigenvalue problems.
\newblock \emph{Linear algebra and its applications}, 431\penalty0
  (3-4):\penalty0 396--408, 2009.

\bibitem[Balakrishnan et~al.(2011)Balakrishnan, Xu, Krishnamurthy, and
  Singh]{balakrishnan2011noise}
S.~Balakrishnan, M.~Xu, A.~Krishnamurthy, and A.~Singh.
\newblock Noise thresholds for spectral clustering.
\newblock \emph{Advances in Neural Information Processing Systems}, 24, 2011.

\bibitem[Bera et~al.(2019)Bera, Chakrabarty, Flores, and
  Negahbani]{NEURIPS2019_fc192b0c}
S.~Bera, D.~Chakrabarty, N.~Flores, and M.~Negahbani.
\newblock Fair algorithms for clustering.
\newblock In H.~Wallach, H.~Larochelle, A.~Beygelzimer, F.~d\textquotesingle
  Alch\'{e}-Buc, E.~Fox, and R.~Garnett, editors, \emph{Advances in Neural
  Information Processing Systems}, volume~32. Curran Associates, Inc., 2019.
\newblock URL
  \url{https://proceedings.neurips.cc/paper/2019/file/fc192b0c0d270dbf41870a63a8c76c2f-Paper.pdf}.

\bibitem[Chierichetti et~al.(2017)Chierichetti, Kumar, Lattanzi, and
  Vassilvitskii]{chierichetti2017fair}
F.~Chierichetti, R.~Kumar, S.~Lattanzi, and S.~Vassilvitskii.
\newblock Fair clustering through fairlets.
\newblock \emph{Advances in Neural Information Processing Systems}, 30, 2017.

\bibitem[Chodrow et~al.(2021)Chodrow, Veldt, and Benson]{chodrow2021generative}
P.~S. Chodrow, N.~Veldt, and A.~R. Benson.
\newblock Generative hypergraph clustering: From blockmodels to modularity.
\newblock \emph{Science Advances}, 7\penalty0 (28):\penalty0 eabh1303, 2021.

\bibitem[Chouldechova and Roth(2018)]{chouldechova2018frontiers}
A.~Chouldechova and A.~Roth.
\newblock The frontiers of fairness in machine learning.
\newblock \emph{arXiv preprint arXiv:1810.08810}, 2018.

\bibitem[Crawford and Milenkovi{\'c}(2018)]{crawford2018cluenet}
J.~Crawford and T.~Milenkovi{\'c}.
\newblock Cluenet: Clustering a temporal network based on topological
  similarity rather than denseness.
\newblock \emph{PloS one}, 13\penalty0 (5):\penalty0 e0195993, 2018.

\bibitem[Davidson and Ravi(2020)]{davidson2020making}
I.~Davidson and S.~S. Ravi.
\newblock Making existing clusterings fairer: Algorithms, complexity results
  and insights.
\newblock In \emph{Proceedings of the AAAI Conference on Artificial
  Intelligence}, pages 3733--3740, 2020.

\bibitem[Deadman et~al.(2012)Deadman, Higham, and Ralha]{deadman2012blocked}
E.~Deadman, N.~J. Higham, and R.~Ralha.
\newblock Blocked schur algorithms for computing the matrix square root.
\newblock In \emph{International Workshop on Applied Parallel Computing}, pages
  171--182. Springer, 2012.

\bibitem[Dwork et~al.(2012)Dwork, Hardt, Pitassi, Reingold, and
  Zemel]{dwork2012fairness}
C.~Dwork, M.~Hardt, T.~Pitassi, O.~Reingold, and R.~Zemel.
\newblock Fairness through awareness.
\newblock In \emph{Proceedings of the 3rd innovations in theoretical computer
  science conference}, pages 214--226, 2012.

\bibitem[Fan(1949)]{fan1949theorem}
K.~Fan.
\newblock On a theorem of weyl concerning eigenvalues of linear transformations
  i.
\newblock \emph{Proceedings of the National Academy of Sciences of the United
  States of America}, 35\penalty0 (11):\penalty0 652, 1949.

\bibitem[Feldman et~al.(2015)Feldman, Friedler, Moeller, Scheidegger, and
  Venkatasubramanian]{feldman2015certifying}
M.~Feldman, S.~A. Friedler, J.~Moeller, C.~Scheidegger, and
  S.~Venkatasubramanian.
\newblock Certifying and removing disparate impact.
\newblock In \emph{proceedings of the 21th ACM SIGKDD international conference
  on knowledge discovery and data mining}, pages 259--268, 2015.

\bibitem[Flores et~al.(2016)Flores, Bechtel, and Lowenkamp]{flores2016false}
A.~W. Flores, K.~Bechtel, and C.~T. Lowenkamp.
\newblock False positives, false negatives, and false analyses: A rejoinder to
  machine bias: There's software used across the country to predict future
  criminals. and it's biased against blacks.
\newblock \emph{Fed. Probation}, 80:\penalty0 38, 2016.

\bibitem[Golub(1973)]{Golub:1973}
G.~H. Golub.
\newblock Some modified matrix eigenvalue problems.
\newblock \emph{SIAM Review}, 15\penalty0 (2):\penalty0 318--334, 1973.

\bibitem[Golub and Van~Loan(1996)]{golub1996matrix}
G.~H. Golub and C.~F. Van~Loan.
\newblock \emph{Matrix computations}.
\newblock Johns Hopkins University Press, 1996.

\bibitem[Golub et~al.(2000)Golub, Zhang, and Zha]{Golub:2000}
G.~H. Golub, Z.~Zhang, and H.~Zha.
\newblock Large sparse symmetric eigenvalue problems with homogeneous linear
  constraints: the {L}anczos process with inner--outer iterations.
\newblock \emph{Linear Algebra And Its Applications}, 309\penalty0
  (1-3):\penalty0 289--306, 2000.

\bibitem[Gupta and Dukkipati(2022)]{gupta2022consistency}
S.~Gupta and A.~Dukkipati.
\newblock Consistency of constrained spectral clustering under graph induced
  fair planted partitions.
\newblock In A.~H. Oh, A.~Agarwal, D.~Belgrave, and K.~Cho, editors,
  \emph{Advances in Neural Information Processing Systems}, 2022.
\newblock URL \url{https://openreview.net/forum?id=FHgpw2Cn__}.

\bibitem[Hardt et~al.(2016)Hardt, Price, and Srebro]{hardt2016equality}
M.~Hardt, E.~Price, and N.~Srebro.
\newblock Equality of opportunity in supervised learning.
\newblock \emph{Advances in neural information processing systems}, 29, 2016.

\bibitem[Higham and Al-Mohy(2010)]{higham2010computing}
N.~J. Higham and A.~H. Al-Mohy.
\newblock Computing matrix functions.
\newblock \emph{Acta Numerica}, 19:\penalty0 159--208, 2010.

\bibitem[Holland et~al.(1983)Holland, Laskey, and Leinhardt]{HOLLAND1983109}
P.~W. Holland, K.~B. Laskey, and S.~Leinhardt.
\newblock Stochastic blockmodels: First steps.
\newblock \emph{Social Networks}, 5\penalty0 (2):\penalty0 109--137, 1983.
\newblock ISSN 0378-8733.
\newblock \doi{https://doi.org/10.1016/0378-8733(83)90021-7}.
\newblock URL
  \url{https://www.sciencedirect.com/science/article/pii/0378873383900217}.

\bibitem[Horn and Johnson(2012)]{horn2012matrix}
R.~A. Horn and C.~R. Johnson.
\newblock \emph{Matrix analysis}.
\newblock Cambridge university press, 2012.

\bibitem[Hotelling(1943)]{hotelling1943some}
H.~Hotelling.
\newblock Some new methods in matrix calculation.
\newblock \emph{The Annals of Mathematical Statistics}, 14\penalty0
  (1):\penalty0 1--34, 1943.

\bibitem[Kleindessner et~al.(2019)Kleindessner, Samadi, Awasthi, and
  Morgenstern]{kleindessner2019guarantees}
M.~Kleindessner, S.~Samadi, P.~Awasthi, and J.~Morgenstern.
\newblock Guarantees for spectral clustering with fairness constraints.
\newblock In \emph{International Conference on Machine Learning}, pages
  3458--3467. PMLR, 2019.

\bibitem[Lang(2005)]{lang2005fixing}
K.~Lang.
\newblock Fixing two weaknesses of the spectral method.
\newblock \emph{Advances in Neural Information Processing Systems}, 18, 2005.

\bibitem[Lehoucq et~al.(1998)Lehoucq, Sorensen, and Yang]{lehoucq1998arpack}
R.~B. Lehoucq, D.~C. Sorensen, and C.~Yang.
\newblock \emph{ARPACK users' guide: solution of large-scale eigenvalue
  problems with implicitly restarted Arnoldi methods}.
\newblock SIAM, 1998.

\bibitem[Lei and Rinaldo(2015)]{lei2015consistency}
J.~Lei and A.~Rinaldo.
\newblock Consistency of spectral clustering in stochastic block models.
\newblock \emph{The Annals of Statistics}, 43\penalty0 (1):\penalty0 215--237,
  2015.

\bibitem[Lin et~al.(2021)Lin, Lu, and Bai]{Lin:2021}
C.-P. Lin, D.~Lu, and Z.~Bai.
\newblock Backward stability of explicit external deflation for the symmetric
  eigenvalue problem.
\newblock \emph{arXiv preprint arXiv:2105.01298}, 2021.

\bibitem[Mastrandrea et~al.(2015)Mastrandrea, Fournet, and
  Barrat]{10.1371/journal.pone.0136497}
R.~Mastrandrea, J.~Fournet, and A.~Barrat.
\newblock Contact patterns in a high school: A comparison between data
  collected using wearable sensors, contact diaries and friendship surveys.
\newblock \emph{PLOS ONE}, 10:\penalty0 1--26, 09 2015.
\newblock \doi{10.1371/journal.pone.0136497}.
\newblock URL \url{https://doi.org/10.1371/journal.pone.0136497}.

\bibitem[Mehrabi et~al.(2021)Mehrabi, Morstatter, Saxena, Lerman, and
  Galstyan]{mehrabi2021survey}
N.~Mehrabi, F.~Morstatter, N.~Saxena, K.~Lerman, and A.~Galstyan.
\newblock A survey on bias and fairness in machine learning.
\newblock \emph{ACM Computing Surveys (CSUR)}, 54\penalty0 (6):\penalty0 1--35,
  2021.

\bibitem[Ng et~al.(2001)Ng, Jordan, and Weiss]{ng2001spectral}
A.~Ng, M.~Jordan, and Y.~Weiss.
\newblock On spectral clustering: Analysis and an algorithm.
\newblock \emph{Advances in neural information processing systems}, 14, 2001.

\bibitem[Paige and Saunders(1982)]{paige1982lsqr}
C.~C. Paige and M.~A. Saunders.
\newblock {LSQR}: An algorithm for sparse linear equations and sparse least
  squares.
\newblock \emph{ACM Transactions on Mathematical Software (TOMS)}, 8\penalty0
  (1):\penalty0 43--71, 1982.

\bibitem[Parlett(1998)]{Parlett:1998}
B.~N. Parlett.
\newblock \emph{The symmetric eigenvalue problem}.
\newblock SIAM, 1998.

\bibitem[Pethig and Kroenung(2022)]{pethig2022biased}
F.~Pethig and J.~Kroenung.
\newblock Biased humans,(un) biased algorithms?
\newblock \emph{Journal of Business Ethics}, pages 1--16, 2022.

\bibitem[Porcelli et~al.(2015)Porcelli, Binante, Girardi, Padovani, and
  Pasquinelli]{porcelli2015solution}
M.~Porcelli, V.~Binante, M.~Girardi, C.~Padovani, and G.~Pasquinelli.
\newblock A solution procedure for constrained eigenvalue problems and its
  application within the structural finite-element code {NOSA-ITACA}.
\newblock \emph{Calcolo}, 52\penalty0 (2):\penalty0 167--186, 2015.

\bibitem[Rohe et~al.(2011)Rohe, Chatterjee, and Yu]{rohe2011spectral}
K.~Rohe, S.~Chatterjee, and B.~Yu.
\newblock Spectral clustering and the high-dimensional stochastic blockmodel.
\newblock \emph{The Annals of Statistics}, 39\penalty0 (4):\penalty0
  1878--1915, 2011.

\bibitem[Rozemberczki and Sarkar(2020)]{feather}
B.~Rozemberczki and R.~Sarkar.
\newblock {Characteristic Functions on Graphs: Birds of a Feather, from
  Statistical Descriptors to Parametric Models}.
\newblock In \emph{Proceedings of the 29th ACM International Conference on
  Information and Knowledge Management (CIKM '20)}, page 1325–1334. ACM,
  2020.

\bibitem[Samadi et~al.(2018)Samadi, Tantipongpipat, Morgenstern, Singh, and
  Vempala]{samadi2018price}
S.~Samadi, U.~Tantipongpipat, J.~H. Morgenstern, M.~Singh, and S.~Vempala.
\newblock The price of fair pca: One extra dimension.
\newblock \emph{Advances in neural information processing systems}, 31, 2018.

\bibitem[Sarkar and Bickel(2015)]{sarkar2015role}
P.~Sarkar and P.~J. Bickel.
\newblock Role of normalization in spectral clustering for stochastic
  blockmodels.
\newblock \emph{The Annals of Statistics}, 43\penalty0 (3):\penalty0 962--990,
  2015.

\bibitem[Shi and Malik(2000)]{shi2000normalized}
J.~Shi and J.~Malik.
\newblock Normalized cuts and image segmentation.
\newblock \emph{IEEE Transactions on pattern analysis and machine
  intelligence}, 22\penalty0 (8):\penalty0 888--905, 2000.

\bibitem[Simoncini(2003)]{simoncini2003algebraic}
V.~Simoncini.
\newblock Algebraic formulations for the solution of the nullspace-free
  eigenvalue problem using the inexact shift-and-invert lanczos method.
\newblock \emph{Numerical linear algebra with applications}, 10\penalty0
  (4):\penalty0 357--375, 2003.

\bibitem[Tung et~al.(2010)Tung, Wong, and Clausi]{tung2010enabling}
F.~Tung, A.~Wong, and D.~A. Clausi.
\newblock Enabling scalable spectral clustering for image segmentation.
\newblock \emph{Pattern Recognition}, 43\penalty0 (12):\penalty0 4069--4076,
  2010.

\bibitem[Wagner and Wagner(1993)]{wagner1993between}
D.~Wagner and F.~Wagner.
\newblock Between min cut and graph bisection.
\newblock In \emph{International Symposium on Mathematical Foundations of
  Computer Science}, pages 744--750. Springer, 1993.

\bibitem[Yamazaki et~al.(2019)Yamazaki, Bai, Lu, and Dongarra]{Yamazaki:2019}
I.~Yamazaki, Z.~Bai, D.~Lu, and J.~Dongarra.
\newblock Matrix powers kernels for thick-restart lanczos with explicit
  external deflation.
\newblock In \emph{2019 IEEE International Parallel and Distributed Processing
  Symposium (IPDPS)}, pages 472--481. IEEE, 2019.

\bibitem[Zemel et~al.(2013)Zemel, Wu, Swersky, Pitassi, and
  Dwork]{zemel2013learning}
R.~Zemel, Y.~Wu, K.~Swersky, T.~Pitassi, and C.~Dwork.
\newblock Learning fair representations.
\newblock In \emph{International conference on machine learning}, pages
  325--333. PMLR, 2013.

\bibitem[Zhang et~al.(2019)Zhang, Basu, and Davidson]{zhang2019framework}
H.~Zhang, S.~Basu, and I.~Davidson.
\newblock A framework for deep constrained clustering-algorithms and advances.
\newblock In \emph{Joint European Conference on Machine Learning and Knowledge
  Discovery in Databases}, pages 57--72. Springer, 2019.

\end{thebibliography}

% If you have textual supplementary material
\appendix
\onecolumn

\section{MISSING PROOFS}

In this section, we provide proofs that are missing in the main manuscript.

\subsection{Proof of Lemma 2.1} \label{appx-le2-1}
For~\eqref{i:le:1:1}:
recall that each row of $G$ contains exactly one nonzero entry, and it equals $1$.
Hence, 
	\begin{equation}\label{eq:g1}
		G {\bf 1}_h = {\bf 1}_n
		\quad\text{and}\quad
		{\bf 1}_n^T G {\bf 1}_h = n.
	\end{equation}
Since $G$ has orthogonal columns, the first equation above implies 
	\begin{equation}\label{eq:g2}
		{\bf 1}_h = G^{\dag} {\bf 1}_n,
	\end{equation}
where $G^{\dag}:=(G^TG)^{-1}G^T$ denotes the pseudo inverse of $G$.
For $\rank(F_0)=h-1$, it is sufficient to show that the nullspace 
of $F_0$ is of dimension one.
Let $x\in\mathbb R^p$ be a null vector of $F_0$, i.e., $F_0x = 0$.
By the definition of $F_0$, we have 
\[
Gx= \alpha \cdot {\bf 1}_n \quad\text{with}\quad \alpha:=({\bf 1}_n^TG x)/n.
\]
A multiplication of $G^{\dag}$ to the equation, together
with~\eqref{eq:g2}, leads to $x = \alpha {\bf 1}_h$.
On the other hand, ${\bf 1}_h$ is a null vector of $F_0$:
\[
F_0\cdot {\bf 1}_h = G {\bf 1}_h - {\bf 1}_n ({\bf 1}_n^TG {\bf
1}_h)/n = 0,
\]
where we used~\eqref{eq:g1}. 
Consequently, the nullspace of $F_0$ is spanned by 
the vector ${\bf 1}_h$. By the rank-nullity theorem in linear algebra, $\rank(F_0) = h-1$.
It follows from $F_0\cdot {\bf 1}_h=0$ that the 
last column of $F_0$ is a linear
combination of  the first $h-1$ columns.
Consequently, we have $\rank(F_0) = \rank(F)$.

For~\eqref{i:le:1:2}:
it follows from~\eqref{i:le:1:1} that $F_0$ and $F$ have the same
range space. Hence, $F_0^Ty=0$ if and only if $F^Ty=0$.
Then~\eqref{i:le:1:2} follows from~\eqref{eq:f0h}. 

\subsection{Proof of Proposition 3.1} \label{appx:prop3-1}
We just need to verify that~\eqref{eq:blockAp} is an eigenvalue decomposition of the matrix $L^{\rm p}_{\rm n}$. 
	First, since $V$ is a basis of the nullspace of $C^T$ and $U_2$ is a
	basis of the range of $C$, we have
 \[ 
 V^TU_2=0 
 \quad \mbox{and} \quad  
		U_1^TU_2 = Y^T(V^T U_2) = 0.
	\]
	A quick verification shows $U=[U_1,U_2]$ satisfy 
	$U^TU=I_n$. Therefore $U$ is orthogonal.

	On the other hand, it follows from 
    $L^{\rm v}_{\rm n} = V^T L_{\rm n} V$ and $L^{\rm p}_{\rm n} = PL_{\rm n}P$ that
	\[
		L^{\rm p}_{\rm n} = V L^{\rm v}_{\rm n} V^T. 
	\]
	Hence, 
	\[
		 L^{\rm p}_{\rm n}\, U_1 = (V L^{\rm v}_{\rm n} V^T) (VY) 
		 = V(L^{\rm v}_{\rm n} Y) = V (Y \Lambda_{\rm v}) = U_1\Lambda_{\rm v},
	\]
	where we used~\eqref{eq:eiglv} in the third equality. 
	On the other hand, $V^TU_2=0$  implies
	\[
		L^{\rm p}_{\rm n}\, U_2 = VL^{\rm v}_{\rm n}V^TU_2 = 0.
	\]
	Consequently, $L^{\rm p}_{\rm n} [U_1,U_2] = [U_1,U_2]\cdot
	\mbox{blockdiag}(\Lambda_{\rm v},
	{\bf 0}_{h-1,h-1})$, i.e., the eigenvalue decomposition ~\eqref{eq:blockAp} of $L^{\rm p}_{\rm n}$.

\subsection{Proof of Proposition 3.2} \label{appx:prop3-2}
    It follows from~\eqref{eq:meig} that 
	$A Q_1= Q_1\Lambda_1$ and $AQ_2= Q_2\Lambda_2$.
	Consequently, 
	$A_\sigma Q_1 = AQ_1+\sigma Q_1
	= Q_1(\Lambda_1+ \sigma I)$
	and 
	$A_\sigma Q_2 = AQ_2 = Q_2\Lambda_2$.

\section{ADDITIONAL EXPERIMENTS AND DISCUSSIONS}

\subsection{m-SBM Dataset} \label{appx:msbm}
In this section, we first describe the standard stochastic block model (SBM). Then, we discuss a modification to accommodate group fairness.

\subsubsection{SBM}
In an SBM with $n$ vertices and $k$ blocks, each vertex is assigned to one block ({\it a cluster}) to prescribe a clustering, and edges are placed between vertex pairs with probabilities 
dependent only on the block membership of the vertices.

Following \citep{lei2015consistency},
to generate a random graph $\mathcal{G}(V, W)$ 
with a ground-truth clustering  $V = C_1 \cup \cdots \cup C_k$ by an SBM, 
we need a pair of parameters $(u, P)$. The vector
$u = [u_1,\, u_2,\, \ldots, u_k] \in \mathbb{N}^k$ 
stores the sizes of each block,
i.e., each element $u_i$ denotes the number of vertices in $C_i$,
so that $\sum^{k}_{i=1} u_i = n$, where $n = |V|$. Once $u$ is assigned, we can represent the ground-truth clustering
by an indicator matrix $H \in \{0,1\}^{n \times k}$
as  defined in~\eqref{eq:H}.
$P$ is a symmetric probability matrix $P = (p_{ij}) \in \mathbb{R}^{n \times n}$ defining the edge connectivity with 
\begin{equation} \label{eq:p-tsbm}
	p_{ij} = \left\{
    \begin{array}{ll}
          a, & \mbox{if $v_i$ and $v_j$ are in the same cluster}, \\
          b, & \mbox{if $v_i$ and $v_j$ are in different clusters}.
     \end{array}
     \right.
\end{equation}
We require $a > b$ so that two vertices within a same cluster have a 
higher chance to be joined by an edge than between clusters.

%Once $(u, P)$ is set, 
%let $\alpha$ be the weight of within-cluster edges and $\beta$ be the weight of between-cluster edges, and $\alpha > \beta$, then  

Next, let $\alpha$ and $\beta$ be the weights for within-cluster and 
between-cluster edges, respectively. 
Then  the weighted adjacency matrix $W = (w_{ij}) \in \{0,\alpha,\beta\}^{n \times n}$ of graph $\mathcal{G}$
is generated by 
\begin{equation} \label{eq:w-tsbm}
w_{ij} = \left\{
\begin{array}{ll}
\mbox{Bernoulli}(p_{ij}), & \mbox{if $i \neq j$}, \\
0, & \mbox{if $i = j$},
\end{array}
\right.
\end{equation}
where $\mbox{Bernoulli}(p_{ij})$ is a random variable satisfying the Bernoulli distribution with probability $p_{ij}$ such that 
\begin{equation} \label{eq:sbm-bernoulli}
\left\{
\begin{array}{ll}
P_r(w_{ij} = \alpha) = p_{ij} = 1 - P_r(w_{ij} = 0), \quad \mbox{if $v_i$ and $v_j$ are in the same cluster}, \\
P_r(w_{ij} = \beta) = p_{ij} = 1 - P_r(w_{ij} = 0), \quad \mbox{if $v_i$ and $v_j$ are in different clusters}. 
\end{array}
\right.
\end{equation}
The SBM graph is then given by $\mathcal{G}(V, W)$.

\begin{example} \label{eg:sbm}
Let $k = 3, n = 6$, we set vector $u = [u_1, u_2, u_3] = [2, 2, 2]$, parameters $\alpha =3, \beta = 1$ for the edge weight, and parameters $a = 0.6, b = 0.2$ for the probability matrix $P$. 
% \Red{Question: why don't require $a+b = 1$?}
    First, we define a ground-truth clustering from $u$ using the cluster indicator matrix $H$ as follows
\begin{align*}
    H &= \begin{bmatrix}
        1 & 1 & 0 & 0 & 0 & 0 \\
        0 & 0 & 1 & 1 & 0 & 0 \\
        0 & 0 & 0 & 0 & 1 & 1 \\
    \end{bmatrix}^{T}.
\end{align*}
Based on $H$, we then use~\eqref{eq:p-tsbm}, ~\eqref{eq:w-tsbm}
and~\eqref{eq:sbm-bernoulli}
to generate the adjacency matrix $W$ with given $\alpha, \beta$, and $a, b$. 
Figure~\ref{fig:SBM-example} depicts the SBM $\mathcal{G}(V, W)$ and the underlying ground-truth clustering $V = C_1 \cup C_2 \cup C_3$. $\Box$

\begin{figure}[ht!]
    \begin{center} 
    \includegraphics[width=0.3\textwidth]{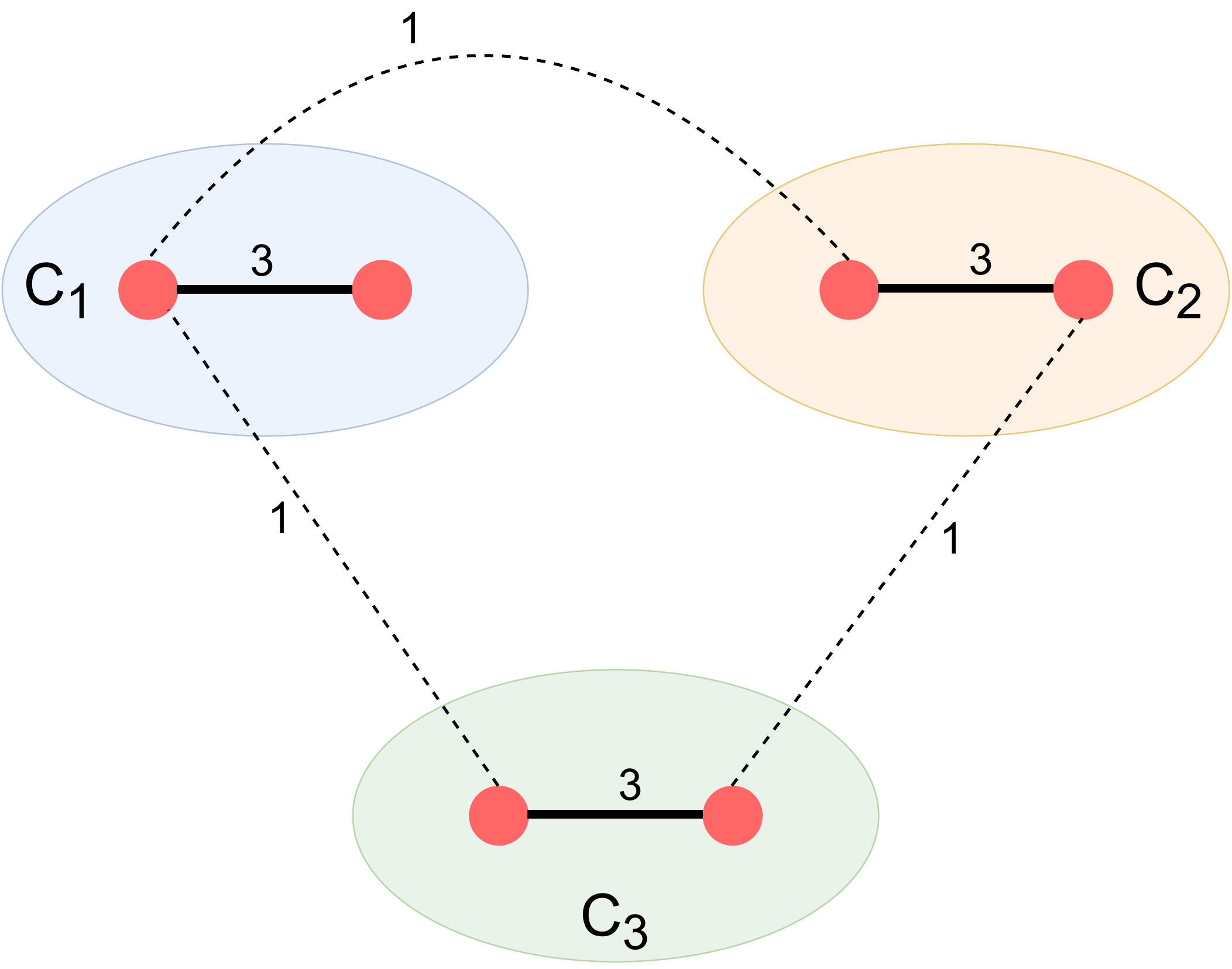} 
    \end{center} 
    \caption{
    The SBM $\mathcal{G}(V, W)$ and the ground-truth clustering $V = C_1 \cup C_2 \cup C_3$ (Example~\ref{eg:sbm}).
    }
\label{fig:SBM-example}
\end{figure}
\end{example}

\subsubsection{m-SBM}
To take group fairness into account, 
we utilize a modified SBM (m-SBM) proposed 
in~\cite{kleindessner2019guarantees}. 
In m-SBM, $n$ vertices are partitioned into $h$ disjoint groups such that
$V = V_{1}\cup \cdots\cup V_{h}$, and 
are assigned to $k$ clusters 
for a prescribed fair ground-truth clustering
$V = C_1 \cup \cdots \cup C_k$. 
An edge between a pair of vertices 
is placed  with probabilities dependent only on 
whether the terminal vertices belong to the same cluster or group.

Specifically, let us first define $u = [u_1,\, u_2,\, \ldots, u_k]$ 
as the block vector consisting of $k$ blocks. 
Each block $u_i \in \mathbb{N}^h$ contains $h$ elements,
and
each element $u_{i}^{(j)}$ of $u_i$ equals the number of vertices in $V_j \cap C_i$, i.e.,
\[
u_{i}^{(j)} = |V_j \cap C_i|.
\]
Consequently, we have 
\[
\sum_{i=1}^{k}u_{i}^{(j)} = |V_{j}|, \quad
\sum_{j=1}^{h}u_{i}^{(j)} = |C_{i}|, 
\quad \mbox{and} \quad
\sum^{k}_{i=1} \sum^{h}_{j=1}u^{(j)}_i = |V|.
\]
To satisfy the fairness condition~\eqref{eq:fair},  
for $j = 1, 2, \cdots, h$, the elements of the vector $u$ should be chosen such that 
\begin{equation} \label{eq:fair-block}
\frac{u_{i}^{(j)}}{\sum_{j=1}^{h}{u_{i}^{(j)}}} 
= \frac{\sum_{i = 1}^{k}{u_{i}^{(j)}}}{|V|}
\quad \mbox{for any } i = 1,2,\cdots, k.
\end{equation}
Once $u$ is set, we have
\begin{itemize}
\item Group-membership vectors $g^{(s)}$, for $s = 1, 2, \cdots, h$,
as  defined in~\eqref{eq:group-vec}, 
and the corresponding group-membership matrix $F$ as defined in Lemma~\ref{le:1};

\item  Clustering indicator matrix $H \in \{0,1\}^{n \times k}$ 
as  defined in~\eqref{eq:H}, containing the fair ground-truth clustering. 
\end{itemize}
The probability matrix 
$P = (p_{ij}) \in \mathbb{R}^{n \times n}$ 
for edge connectivity is defined as follows:
\begin{equation} \label{eq:p-msbm}
p_{ij} = \\
\left\{
\begin{array}{ll}
a, & \mbox{if $v_i, v_j$ are in the same cluster and group}, \\
b, & \mbox{if $v_i, v_j$ are in different clusters but the same group}, \\
c, & \mbox{if $v_i, v_j$ are in the same cluster but different groups}, \\
d, & \mbox{if $v_i, v_j$ are in different clusters and groups},
\end{array}
\right.
\end{equation}
where $a > b > c > d$. 
Let $\alpha$ be the weight of within-cluster edges and $\beta$ be the weight of between-cluster edges and $\alpha > \beta$, then
the adjacency matrix of the m-SBM $\mathcal{G}(V, W)$ is given by 
$W = (w_{ij}) \in \{0,\alpha,\beta\}^{n \times n}$
as follows:
    \begin{equation} \label{eq:w-msbm}
    w_{ij} = \left\{
    \begin{array}{ll}
          \mbox{Bernoulli}(p_{ij}), & \mbox{if $i \neq j$}, \\
          0, & \mbox{if $i = j$},
     \end{array}
     \right.
    \end{equation}
    where $\mbox{Bernoulli}(p_{ij})$ is a random variable satisfying the Bernoulli distribution 
    % of $w_{ij}$ 
    with probability $p_{ij}$ such that 
    \begin{equation} \label{eq:msbm-bernoulli}
    \left\{
    \begin{array}{ll}
    P_r(w_{ij} = \alpha) = p_{ij} = 1 - P_r(w_{ij} = 0), 
    \quad \mbox{if $v_i$ and $v_j$ are in the same cluster}, \\
    P_r(w_{ij} = \beta) = p_{ij} = 1 - P_r(w_{ij} = 0), \quad
    \mbox{if $v_i$ and $v_j$ are in different clusters}. 
    \end{array}
    \right.
    \end{equation}

\begin{example} \label{eg:msbm}
Let $k = 3$, $h = 2$ and $n = 10$, we set the block vector $u = [u_1, u_2, u_3] = [(2, 2), (2, 2), (1, 1)]$, parameters $\alpha =3$
and $\beta = 1$ for the edge weight, and parameters $a = 0.6, b = 0.4, c = 0.2, d = 0.1$ for the probability matrix $P$. 
By the vector $u$, we have the following group membership vectors
$g^{(s)}$ and the indicator matrix $H$ of the fair ground-truth clustering 
\begin{align*}
    g^{(1)} &= \begin{bmatrix}
        1 & 1 & 0 & 0 & 
        1 & 1 & 0 & 0 & 
        1 & 0
    \end{bmatrix}^{T}, \\
    g^{(2)} &= \begin{bmatrix}
        0 & 0 & 1 & 1 & 
        0 & 0 & 1 & 1 & 
        0 & 1
    \end{bmatrix}^{T}, \\
    H &= \begin{bmatrix}
        1 & 1 & 1 & 1 & 0 & 0 & 0 & 0 & 0 & 0 \\
        0 & 0 & 0 & 0 & 1 & 1 & 1 & 1 & 0 & 0 \\
        0 & 0 & 0 & 0 & 0 & 0 & 0 & 0 & 1 & 1 \\
    \end{bmatrix}^{T}.
\end{align*}
By~\eqref{eq:p-msbm},~\eqref{eq:w-msbm} and~\eqref{eq:msbm-bernoulli}, with given weights $\alpha, \beta$, probabilities $a, b, c, d$, group information $g^{(1)}, g^{(2)}$, and the cluster information in $H$,
we can generate the adjacency matrix $W$. 
The m-SBM $\mathcal{G}(V, W)$ is shown in Figure~\ref{fig:MSBM-example}.  $\Box$

\begin{figure}[ht!]
    \begin{center} 
    \includegraphics[width=0.35\textwidth]{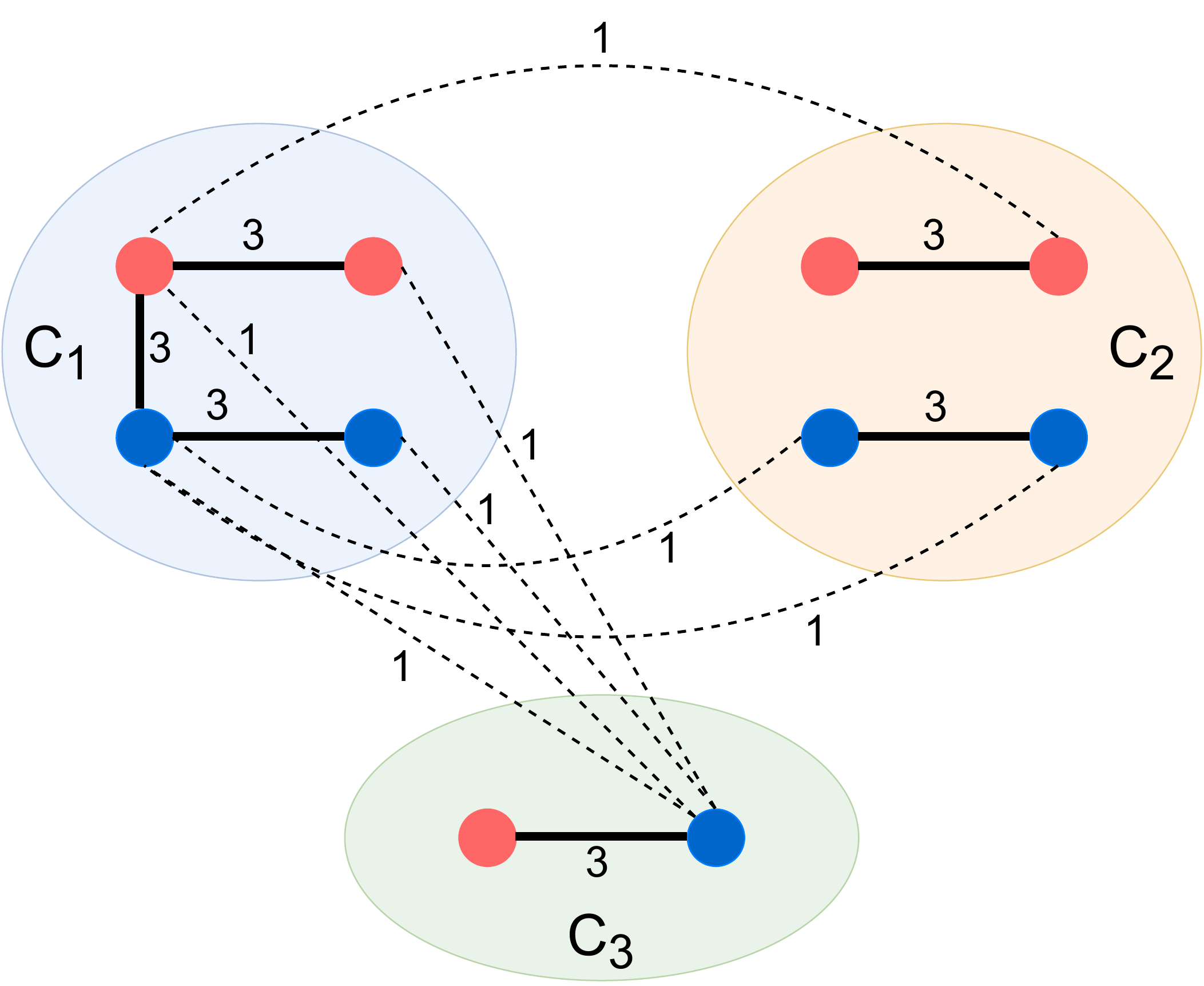} 
    \end{center} 
    \caption{
    An m-SBM $\mathcal{G}(V, W)$ and fair ground-truth clustering $V = C_1 \cup C_2 \cup C_3$ with respect to the group partition $V = V_1 \cup V_2$, where $V_1$ is the set of red vertices and $V_2$ is the set of blue vertices (Example~\ref{eg:msbm}).
    }
\label{fig:MSBM-example}
\end{figure}
\end{example}

\subsection{Fairness Measured by Balance} \label{appx:bal}
For the metric of balance discussed in Section~\ref{sec:ex-results}, we claimed ``a higher balance implies a fairer clustering''.
Here we give a brief justification.
First, let us recall the definitions of balance and average balance.
\begin{definition}
Given a clustering $V = C_{1}\cup\cdots\cup C_{k}$ and a non-overlapping group partition $V = V_1 \cup \cdots \cup V_h$, 
the {\em balance} of cluster $C_{\ell}$ 
for $\ell = 1, 2, \cdots, k$ is defined as  
\begin{equation} \tag{\ref{eq:bal}}
    \mbox{Balance}(C_{\ell}) := \min_{s \neq s' \in \{1, \cdots, h\}} \frac{|V_{s}\cap C_{\ell}|}{|V_{s'}\cap C_{\ell}|} .
\end{equation}
The {\em average balance}  is defined as 
\begin{equation} \tag{\ref{eq:avebal}}
\mbox{Average\_Balance} = \frac{1}{k}\sum_{\ell=1}^{k} \text{Balance}(C_{\ell}). 
\end{equation} 
\end{definition} 

It follows from~\cite{kleindessner2019guarantees} 
that for any clustering, we have 
\begin{equation} \label{eq:bal-upper}
\min_{\ell \in \{1, \cdots, k\}} \mbox{Balance}(C_{\ell}) \leq 
\min_{s \neq s' \in \{1, \cdots, h\}} \frac{|V_s|}{|V_{s'}|}.
\end{equation}
Assume a clustering reaches the upper bound \eqref{eq:bal-upper}, i.e., 
\begin{equation} \label{eq:gfair-equiv}
\frac{|V_{s}\cap C_{\ell}|}{|V_{s'}\cap C_{\ell}|}
= \frac{|V_s|}{|V_{s'}|},
\end{equation} 
for $\ell = 1, \ldots, k$ and 
$s,s' = 1, \ldots, h$ and $s \neq s'$.
Then we will have the group fairness as defined in \eqref{eq:fair}, i.e.,
\begin{equation}
\frac{|V_{s}\cap C_{\ell}|}{|C_{\ell}|} 
= \frac{|V_{s}|}{|V|},
\end{equation} 
for $\ell = 1, \ldots, k$ and $s = 1, \ldots, h$. 
To justify the equation above, 
 we derive from \eqref{eq:gfair-equiv} that
\begin{equation}
    |V_{s}\cap C_{\ell}| \cdot |V_{s'}| = |V_{s}| \cdot |V_{s'}\cap C_{\ell}|.
\end{equation}
Therefore, 
\[
    \sum_{s \neq s' \in \{1, \dots, h\}} |V_{s}\cap C_{\ell}| \cdot |V_{s'}| = 
    \sum_{s \neq s' \in \{1, \dots, h\}} |V_{s}| \cdot |V_{s'}\cap C_{\ell}|. 
\]
Consequently, 
\[
(|C_\ell| - |V_{s'}\cap C_{\ell}|) |V_{s'}| = (|V| - |V_{s'}|) |V_{s'}\cap C_{\ell}|,
\]
or equivalently
\[
\frac {|V_{s'}\cap C_{\ell}|} {|C_{\ell}|}
= \frac{|V_{s'}|}{|V|}.
\]
Multiplication of the above equation to \eqref{eq:gfair-equiv} 
leads directly to the group fairness \eqref{eq:fair}:
\[
\frac{|V_{s}\cap C_{\ell}|}{|C_{\ell}|} 
= \frac{|V_{s}|}{|V|}.
\]
Combining~\eqref{eq:bal}, ~\eqref{eq:bal-upper}, and~\eqref{eq:gfair-equiv}, 
we can conclude that a higher value of Average\_Balance indicates a fairer clustering.

%Using \eqref{eq:gfair-equiv} again, we 
%have
%\[
%\frac{|C_{\ell}|} {|V_{s'}\cap C_{\ell}|} \cdot \frac{|V_{s'}\cap C_{\ell}|}{|V_{s}\cap C_{\ell}|}
%= \frac{|V|}{|V_{s'}|} \cdot 
%\frac{|V_{s'}|}{|V_{s}|},
%\]
%i.e., we have the group fairness \eqref{eq:fair}:
%\[
%\frac{|V_{s}\cap C_{\ell}|}{|C_{\ell}|} 
%= \frac{|V_{s}|}{|V|}.
%\]
%Combining~\eqref{eq:bal}, ~\eqref{eq:bal-upper}, and~\eqref{eq:gfair-equiv}, 
%we can conclude that a higher value of Average\_Balance indicates a fairer clustering.
%
\subsection{Hyperparameter Tuning}
s-FairSC (Algorithm~\ref{code:algo3}) takes a few parameters as input, namely, parameter $k$ (the number of clusters), parameter $h$ (the number of groups), and shift parameter $\sigma$ for Hotelling’s deflation. In this section, we illustrate the effect of these parameters on the performance of our algorithm.

\paragraph{Parameter $k$.} 
In Experiment 4 on Random Laplacian dataset, we report the performance with respect to the parameter $k$; see Figure~\ref{fig:time_lapla}. 

\paragraph{Parameter $h$.}
We provide additional experiments on the performance with respect to the parameter $h$.
Recall that the number of groups $h$ is associated with the number of linear constraints in the fair SC model~\eqref{eq:normal-prob-fair}.
We perform experiments with fixed probabilities $a, b, c, d$ and number of clusters $k$, but different group numbers $h$.
Figure~\ref{fig:SBM-h} depicts 
the error rates and running time of SC, FairSC, and s-FairSC. 
SC and s-FairSC are tested for model sizes from $n = 1000$ to $10000$. However, 
FairSC stops at $n = 4000$ due to its high computational cost, echoing results reported in~\cite{kleindessner2019guarantees}.

From Figure~\ref{fig:benchmarkSBM}, we observe that SC has high error rates and fails to recover the fair ground-truth clustering, 
while both FairSC  
and s-FairSC are able to retrieve the
fair ground-truth clustering and s-FairSC is as accurate as FairSC. 
However, the running time of s-FairSC 
is only a fraction of that of FairSC.
For instance, when $n = 4000$, s-FairSC is about 10$\times$ to 12$\times$ faster than FairSC.
We also observe that s-FairSC is as scalable as SC without fairness constraints. Noticeably, we observe that when $h=10$, s-FairSC is even faster than SC. For example, when $n = 10000$, s-FairSC takes only 62\% time of SC. We anticipate that it is due to 
the factors such as the sparsity distribution of the random graph $\mathcal{G}(V,W)$, and
the eigenvalue distribution of $L^{\sigma}_{\rm n}$ in~\eqref{eq:asigma}.
This issue is subject to further investigation.

 \begin{figure}[ht!]
        \begin{center} 
        {\includegraphics[width=0.6\textwidth]{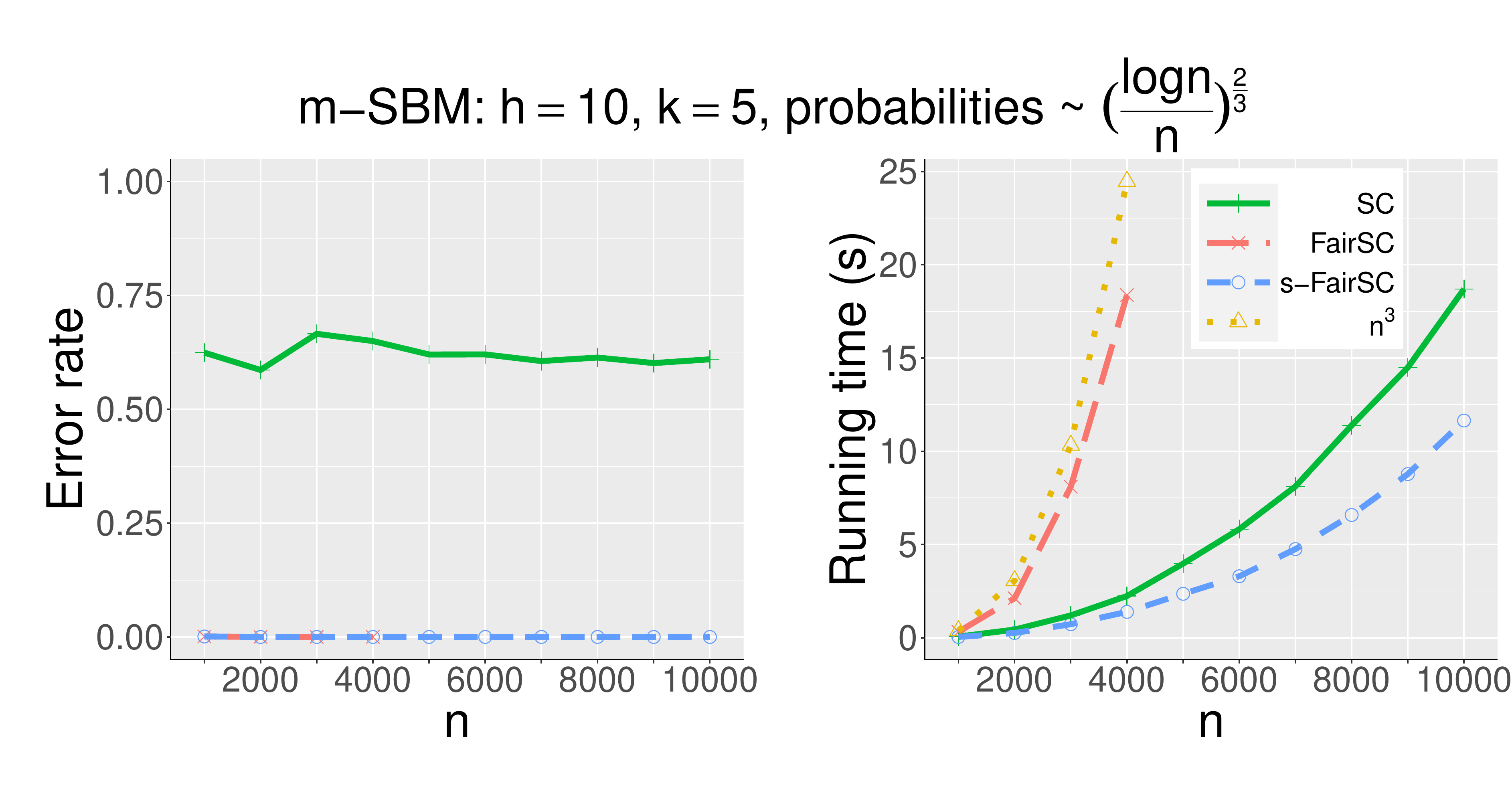}  \\
        \includegraphics[width=0.6\textwidth]{fig/SBM/MSBM_err_time.pdf} \\
        \includegraphics[width=0.6\textwidth]{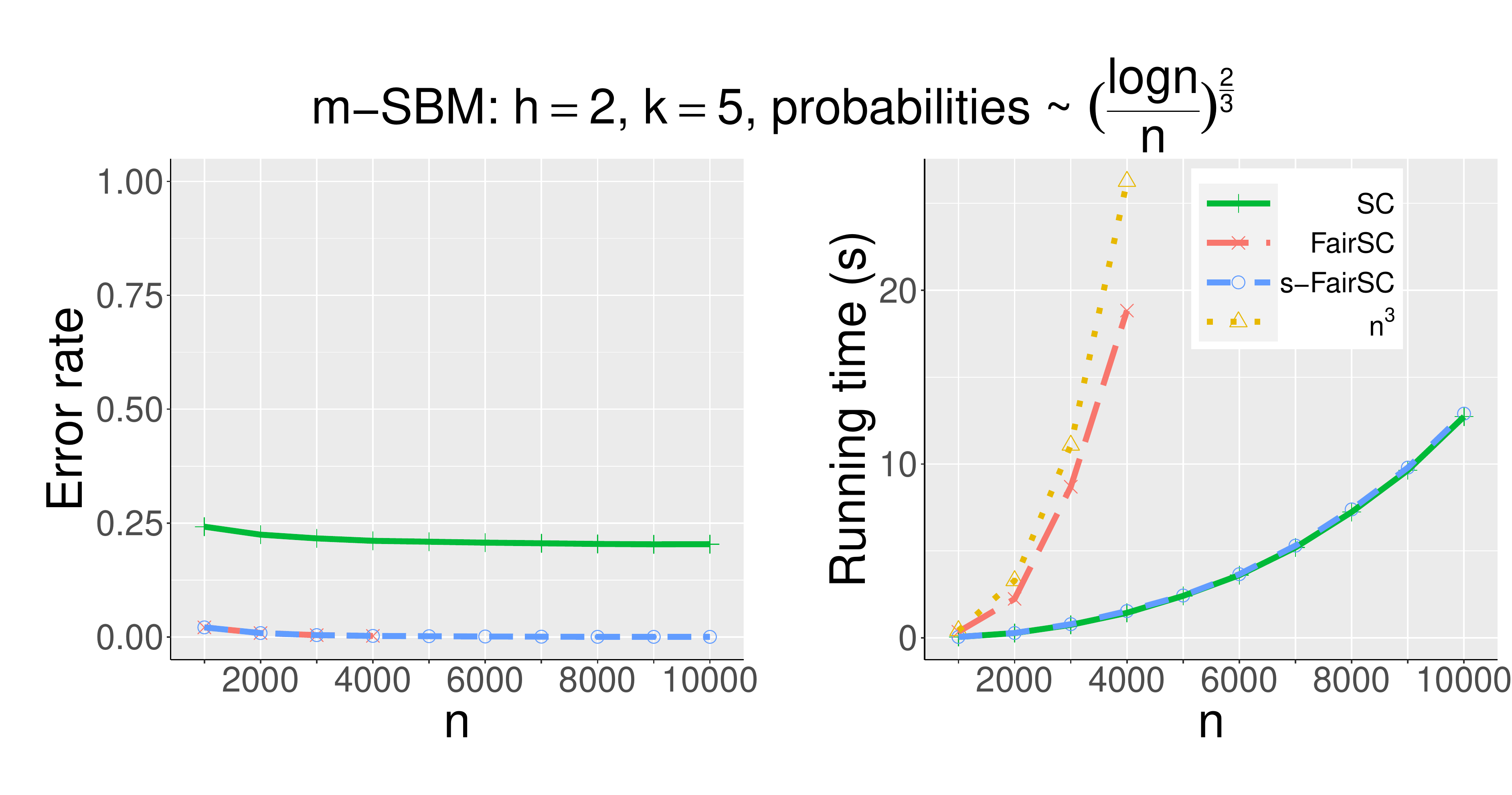}  \\
        }
        \end{center} 
        \caption{Error rates and running time (in seconds) of SC, FairSC, and s-FairSC on m-SBM
    with $h \in \{2,5,10\}$, edge connectivity probabilities proportional 
    to $(\frac{\log n}{n})^\frac{2}{3}$,
    specifically $a = 10\times (\frac{\log n}{n})^{2/3}, 
    b = 7\times (\frac{\log n}{n})^{2/3}, 
    c = 4\times (\frac{\log n}{n})^{2/3}, 
    d = (\frac{\log n}{n})^{2/3}$.
   }
\label{fig:SBM-h}
\end{figure}

\paragraph{Parameter $\sigma$.}
In our implementation, we use $\norm{L_{\rm n}}_1$ as shift $\sigma$. Theoretical justification for the choice of the shift can be found in \cite{Lin:2021}.  Given the equivalence of matrix norms, there is no significant difference with respect to the choice of the matrix norm. 

\end{document}